\documentclass[onefignum,onetabnum,onealgnum, onethmnum]{siamart190516}
\usepackage{csvsimple}
\usepackage{rotating}
\usepackage{subcaption}
\usepackage{url}
\usepackage{bm}
\usepackage{amsfonts}
\usepackage{amsmath}
\usepackage{amssymb}

\newcommand{\beginsupplement}{%
        \setcounter{table}{0}
        \renewcommand{\thetable}{S\arabic{table}}%
        \setcounter{figure}{0}
        \renewcommand{\thefigure}{S\arabic{figure}}%
     }   
     
\usepackage[sort,comma,square,numbers]{natbib}

\usepackage{algorithm}
\usepackage{algpseudocode}

\newcommand{\Linefor}[2]{%
    \State \algorithmicfor\ {#1}\ \algorithmicdo\ {#2} \algorithmicend\ \algorithmicfor%
}

\newtheorem{lem}{Lemma}
\newtheorem{defi}{Definition}

\providecommand{\sct}[1]{{\sc \texttt{#1}}}

\providecommand{\ve}[1]{\boldsymbol{#1}}
\providecommand{\ma}[1]{\boldsymbol{#1}}

\newcommand{\argmax}{\operatornamewithlimits{argmax}}

\newcommand{\T}{^{\ensuremath{\mathsf{T}}}}           
\providecommand{\mc}[1]{\mathcal{#1}}
\providecommand{\mb}[1]{\boldsymbol{#1}}

\providecommand{\mt}[1]{\widetilde{#1}}

\providecommand{\mv}[1]{\vec{#1}}
\providecommand{\mh}[1]{\hat{#1}}

\providecommand{\mtb}[1]{\widetilde{\boldsymbol{#1}}}

\newcommand{\Sporf}{\sct{Sporf}}
\newcommand{\Morf}{\sct{Morf}}
\newcommand{\Mf}{\sct{Morf}}

\newcommand{\Rf}{{\sc \texttt{RF}}}

\newcommand{\Xgboost}{{\sc \texttt{XGBoost}}}

\newcommand{\iid}{\overset{iid}{\sim}}

\newcommand{\Real}{\mathbb{R}}

\newcommand{\conv}{\rightarrow}

\newcommand{\mcE}{\mathcal{E}}

\newcommand{\defeq}{\coloneqq}

\newcommand{\EE}{\mathbb{E}}           

\newcommand{\II}{\mathbb{I}}           

\newcommand{\PP}{\mathbb{P}}

\newcommand{\RR}{\mathbb{R}}

\DeclareMathOperator{\R}{R} 
\DeclareMathOperator{\Y}{\mathbf{Y}}

\DeclareMathOperator{\X}{\mathbf{X}}

\newcommand{\bth}{\ve{\theta}}

\newcommand{\thetn}{\ve{\theta}}
\newcommand{\thet}{\thetn}

\pdfoutput=1

\usepackage[mathlines]{lineno} 
\linenumbers

\title{Manifold Oblique Random Forests: \\ Towards Closing the Gap on Convolutional Deep Networks}

\author{%
    Adam Li$^{*,1,3}$, %
    Ronan Perry$^{*,1}$, %
    Chester Huynh$^{*,1,3}$, %
    Tyler M.~Tomita$^1$, %
    Ronak Mehta$^2$, %
    Jes\'us Arroyo$^2$, %
    Jesse Patsolic$^2$, %
    Benjamin Falk$^2$, %
    Sridevi V. Sarma$^{1,3}$, %
    Joshua T. Vogelstein$^{1,2,3}$%
    \thanks{
        $^1$ Department of Biomedical Engineering,
        $^2$ Center for Imaging Science,
        $^3$ Institute for Computational Medicine,
        Kavli~Neuroscience~Discovery Institute,
        Johns Hopkins University,
        $^*$ Indicates co-first authorship with equal contribution and any ordering of these authors is allowed
    } 
}

\date{June 20, 2019}

\begin{document}

\maketitle

\pagenumbering{arabic}
\setcounter{page}{1}

\begin{abstract}
Decision forests, in particular random forests and gradient boosting trees have demonstrated state-of-the-art accuracy compared to other methods in many supervised learning scenarios. Forests dominate other methods in tabular data, that is, when the feature space is unstructured, so that the signal is invariant to a permutation of the feature indices. However, in structured data lying on a manifold---such as images, and time-series---deep networks, specifically convolutional deep networks (ConvNets), tend to outperform forests. We conjecture that it is in part due to networks not simply analyzing feature magnitudes, but also their indices. In contrast, na\"ive forest implementations fail to explicitly consider feature indices. A recent approach demonstrates that forests, for each node, implicitly sample a random matrix from some specific distribution. These forests, like some networks, learn by partitioning the feature space into convex polytopes corresponding to linear functions. We build on that approach with Manifold Oblique Random Forests (Morf) that chooses distributions in a \emph{manifold-aware fashion} to incorporate feature locality. Morf runs fast and maintains interpretability and theoretical justification. Morf also has excellent empirical classification performance on simulated data and real images and multivariate time-series. It outperforms non-neural network approaches that ignore feature space structure and challenges the performance of ConvNets in some cases.
\end{abstract}

\section{Introduction}
Decision forests, including random forests and gradient boosting trees, have solidified themselves in the past couple decades as a powerful ensemble learning method in supervised settings \cite{JMLR:v15:delgado14a, Caruana:2006:ECS:1143844.1143865}, including both classification and regression \cite{hastie01statisticallearning}. In classification, each forest is a collection of decision trees whose individual classifications of a data point are aggregated together using majority vote. One of the strengths of this approach is that each decision tree need only perform better than chance for the forest to be a strong learner, given a few assumptions \cite{Schapire:1990:SWL:83637.83645,biau_consistency_nodate}. Additionally, decision trees are relatively interpretable because they can provide an understanding of which features are most important for correct classification \cite{Breiman2001}. In 2001, Breiman originally proposed decision trees that partition the data set using hyperplanes aligned to feature axes \cite{Breiman2001}. Yet, this limits the flexibility of the forest and requires trees of large depth to classify some data sets, leading to overfitting. He also suggested that algorithms which partition based on linear combinations of the coordinate axes can improve performance \cite{Breiman2001, Menze2011-obliquerf}, which was corroborated in subsequent work~\cite{tomita2}. More recently, Sparse Projection Oblique Randomer Forest (Sporf)---which leverages sparse random projections of the data---has shown impressive improvement over other methods \cite{SPORF}. Other extensions have led to neural decision forests \cite{kontschieder_deep_2015, biau_neural_2018} which attempt to combine the strengths of neural networks and random forests by using differentiable functions at split nodes and leaves, leading to trees which can be learned via backpropagation. Under certain function choices, once learned, these forests turn out to be equivalent to neural networks with many zeroed weights \cite{biau_neural_2018}.

Random forests and other machine learning algorithms typically operate in a tabular setting, viewing an observation $\bm{x}= (x_1,\ldots,x_p)^T \in \mathbb{R}^p$ as an unstructured feature vector. In doing so, they neglect the feature indices in settings where the indices encode additional information. For structured data, e.g. images or time series, traditional decision forests do not incorporate the known local structure. For decision forests to utilize known local structure in data, new features encoding this information must be manually constructed or new splitting criterion must be implemented. Prior research has extended random forests to a variety of computer vision tasks \cite{rf_keypoint_recog,rf_hough_detection,rf_image_classification,kinect_rf} and augmented random forests with structured pixel label information \cite{rf_structured}. The decision tree at the heart of the Microsoft Kinect showed great success by specializing for image data with depth information \cite{kinect_rf}. Yet these methods either generate features \emph{a priori} from individual pixels (and thus do not take full advantage of the local topology) or lack the flexibility to learn relevant structure. Other approaches have circumvented the problem of learning from raw structured data through tabular feature engineering, notably employed by the aforementioned deep neural decision forest \cite{kontschieder_deep_2015} using a convolutional deep network (ConvNets). Decision forests have also been used to learn distance metrics on unknown manifold structures \cite{Criminisi:2012:DFU:2185837.2185838}, but such manifold forest algorithms are unsupervised.

Inspired by Sporf, we propose a classification algorithm, Manifold Oblique Random Forests (Morf). Morf takes a projection distribution that accounts for neighboring features on a manifold, while incorporating enough randomness to learn the relevant projections. At each node in the decision tree, a set of neighboring features are randomly selected using knowledge of the underlying manifold. Weighting and summing the values of the selected features yields a set of oblique projections of the data which can then be evaluated to partition the observations. We show Morf's effectiveness across simulated and real-data settings as compared to common classification algorithms. In each case, Morf~performs better than non-ConvNet algorithms that lack local feature information, while approaching, or even improving upon ConvNet performance in certain real data applications. Furthermore, the optimized and parallelizable open source implementation of Morf in Python is available at \url{https://neurodata.io/code/}.

\section{Background and Related Work}
We will first define the notation and classification framework needed to describe \Mf.
\subsection{Classification}
    Let $(X,Y) \in \mathcal{X} \times \mathcal{Y}$ be a random sample from the joint distribution $F_{XY}$ and $D_n := \{(x_i,y_i)\}_{i=1}^n$ be our $n$ observed data points where all $(x_i, y_i) \in \mathcal{X} \times \mathcal{Y}$ are drawn from $F_{XY}$. Denote $\mathcal{X} \subseteq \RR^p$ as the space of data vectors, and $\mathcal{Y} = \{1,\dots,K\}$ as the space of K class labels. A classifier is a function that assigns to any unseen data point, $X$, a class label $y \in \mathcal{Y}$. Our goal is to learn a classifier $g_n(X ; D_n): \mathcal{X} \times (\mathcal{X} \times \mathcal{Y})^n \rightarrow \mathcal{Y}$ from our data $D_n$ that minimizes the expected risk corresponding to $0$-$1$ loss, equivalently the probability of incorrect classification,
    
    \begin{equation*}
        L(g) := \EE[\II[g(X) \neq Y]] = P(g(X) \neq Y),
    \end{equation*}
    
    with respect to the distribution of $F_{XY}$. The optimal such classifier is the Bayes classifier
    
    \begin{equation*}
        g^*(X):=\argmax_{y \in \{1,\dots,K\}} P(Y=y \mid X),
    \end{equation*}
    
    which has the lowest attainable risk $L^* := L(g^*(X))$.

\subsection{Random Forests}
    Originally popularized by Amit and Geman~\cite{Amit1997-nd} and subsequently codified by Breiman, the random forest (RF) classifier is empirically very effective~\cite{JMLR:v15:delgado14a} while maintaining strong theoretical guarantees \cite{Breiman2001}. A random forest is an ensemble of decision trees whose individual classifications of a data point are aggregated together using majority vote. Each decision tree recursively partitions the feature space and then makes separate predictions in each of the final subspaces. A partition occurs at a split node in the tree on a subset of the data $S = \{(x_i, y_i)\} \subseteq D_n$. The node is split into two child nodes, each associated with a partition of $S$ based on the value of a selected feature $j \in \{1,\cdots,p\}$. Let $e_j \in  \mathbb{R}^p$ denote a unit vector in the standard basis (that is, a vector with a single one in the $j$th entry and the rest of the entries are zero) and $\tau \in \RR$ a threshold value. Then, $S$ is partitioned into the two subsets, a left node ($L$), and a right node ($R$).
    
    \begin{align*}
        S^L_\theta &= \{(x_i,y_i)\:|\:e_j\T x_i < \tau, (x_i, y_i) \in S\}, \\
        S^R_\theta &= \{(x_i,y_i)\:|\:e_j\T x_i \geq \tau, (x_i, y_i) \in S\}
    \end{align*}
    
    given the parameter pair $\theta = \{e_j, \tau\}$. To choose the partition, $\theta$ is sampled $d$ times (also known as $m_{try}$ in the literature). Then the locally optimal $\theta^*=(e_j^*,\tau^*)$ pair is greedily selected from among a set of $d$ randomly selected standard basis vectors as that which maximizes some measure of information gain. A typical measure is a decrease in impurity, calculated by the Gini impurity score $I(S)$, of the resulting partitions \cite{hastie01statisticallearning}. Let $\hat p_k(S) = \frac{1}{|S|}\sum_{y_i \in S} \II[y_i = k]$ be the fraction of elements of class $k$ in partition $S$ and $I(S) := \sum_{k=1}^{K} \hat p_k(1-\hat p_k)$ be the Gini impurity. Then the split
    
    \begin{align*}
        \theta^* &= \argmax_{\theta} |S| I(S) - |S^L_\theta|I(S^L_\theta) - |S^R_\theta|I(S^R_\theta)
    \end{align*}
    
    is chosen to maximize the decrease in impurity from the parent node containing $S$. A leaf node in the decision tree is created once a partition reaches a stopping criterion, typically either falling below an impurity score threshold or a minimum number of samples \cite{hastie01statisticallearning}.

    To classify a feature vector $x$, it is evaluated at root node of the tree and split into one of the two partitions. This process is repeated recursively at subsequent split nodes until $x$ "falls into" a leaf, upon which posterior probability estimates of the class labels can be assigned. Let $l_b(x)$ be the set of training examples at the leaf node in tree $b$ into which $x$ falls. The empirical posterior probability of label $y$ in $b$ is thus $\hat{p}_{n,b}(y \mid x) = \frac{1}{|l_b(x)|} \sum_{i=1}^n \II[y_i = y]\II[x_i \in l_b(x)]$. The forest composed of $B$ trees computes the empirical posterior probability for $x$ by averaging over the trees $\hat{p}_n(y \mid x) = \frac{1}{B}\sum_{b=1}^B p_{n,b}(y \mid x)$ and classifies $x$ per the label with the greatest empirical posterior probability \cite{hastie01statisticallearning}
    
    \begin{equation*}
        g_n(x) = \argmax_{y \in \{1,\dots,K\}} \hat{p}_n(y | x).
    \end{equation*}

    For good performance of the ensemble, the individual decision trees must be relatively uncorrelated from one another. This is typically done by considering a random subset of features at each split node and training each tree on a bootstrapped subsample of the full training data. Applying these techniques reduces the amount random forests overfit and lowers the upper bound of the generalization error \cite{Breiman2001}.

\subsection{Oblique Forests}
\label{subsec:background_oblique_forests}
    Sparse Projection Oblique Randomer Forests (\Sporf), is a recent modification to random forest that has shown improvement over axis-aligned random forests and other oblique forests that compute linear combinations of features \cite{SPORF,tomita2,Menze2011-obliquerf}.
    Recall that \Rf\ split nodes partition data along the coordinate axes by comparing the projection $e_j\T x$ of observation $x$ on standard basis $e_j$ to a threshold value $\tau$. \Sporf\ generalizes the set of possible projections, allowing for the data to be partitioned along any linear combination of axes specified by the sparse vector $a_j \in \RR^p$. The partition 
    
    \begin{align*}
        S^L_\theta &= \{(x_i,y_i)\:|\:a_j\T x_i < \tau, (x_i,y_i) \in S\}, \\
        S^R_\theta &= \{(x_i,y_i)\:|\:a_j\T x_i \geq \tau, (x_i,y_i) \in S\}
    \end{align*}
    
    follows from our choice of $\theta = \{a_j,\tau\}$, where the entries of $a_j$ vector entries are defined as follows (here  $a_{ij}$ is the ith entry of $a_j$):
    
        \begin{equation*}
            a_{ij} = \begin{cases} 
                1 \quad \text{with prob. $\frac{1}{2s}$} \\
                0 \quad \text{with prob. $1 - \frac{1}{s}$} \\
                -1 \quad \text{with prob. $\frac{1}{2s}$}
            \end{cases}
        \end{equation*}
        
    Then $\theta$ is chosen in the same manner as in axis-aligned Random Forests. All other aspects of \Sporf\ are the same as \Rf.

\section{Methods}
\subsection{Sampling Projections from a Dictionary}
    To move towards manifold forests, we observe that random and oblique forests both sample atoms from a dictionary to create their projected feature values, which are compared with threshold $\tau$ to determine the partition. In axis-aligned random forests, the dictionary, $\mc{A} = \{e_j\}_{j=1}^p$ is the set of points along the p-dimensional hypercube, i.e. standard basis vectors in $\mathbb{R}^p$. Then at every node, atoms $e_i \in \mc{A}$ from the dictionary are sampled $d$ times. 
    
    Similarly, in oblique forests, let the dictionary $\mc{A}$ be the set of vectors (atoms) $\{a_j\}$, each atom a $p$-dimensional vector defining a possible projection $a_j\T x$. In \Sporf, the dictionary $\mc{A}$ can be much larger then that of Random Forests, because it includes, for example, all 2-sparse vectors.
    At each split node, \Sporf\ samples $d$ atoms from $\mc{A}$ according to a specified distribution. By default, each of the $d$ atoms are randomly generated with the number of non-zero elements drawn from a Poisson distribution with a specified rate $s$. Then, each of the non-zero elements are uniformly randomly assigned either $+1$ or $-1$.
    Note that although the size of the dictionary for \Sporf\ is $3^p$ (because each of the $p$ elements could be $-1$, $0$, or $+1$), the atoms are sampled from a distribution heavily skewed towards sparsity controlled by the $s$ term.

\subsection{Random Projection Forests on Manifolds}

In the structured data setting, the dictionary of atoms $\mc{A} = \{a_j\}$ is modified to take advantage of \textit{a priori} knowledge of feature locality on the underlying manifold on which the data lie. We call this `Manifold Oblique Random Forest' (\Mf). This modification constrains the space of random projection decision trees which can be learned in order to better suit certain classification tasks where relations between features may add information. By constructing features in this way, \Mf\ learns low-level features in the structured data, such as corners in images or spikes in time-series.

As in \Sporf, let $\mc{A}$ be a dictionary of $m$, $p$-dimensional atoms with probability density or mass function $f_\mc{A}$ over the $m$ atoms. Each atom $a_j \in \mc{A}$ projects an observation $x_i$ to a real number $a_j^T x_i$, where nonzero elements of $a_j$ effectively weight and sum features.
At each node in the decision tree, \Mf\ selects the best split according to the Gini index over each candidate atom and threshold pair. It has the same partition functions, $S^L_\theta$ and $S^R_\theta$ as SPORF presented in \ref{subsec:background_oblique_forests}. What changes when going from SPORF to MORF? We present two generalizations, selection of the non-zero indices and weights, which allow one to specify any type of manifold structure they want. 

\paragraph{Selection of Nonzero Indices}
    While the nonzero indices of each atom in \Sporf\ were mutually independent of one another, the key aspect of \Mf\ lies in the user-specified restriction of possible atoms to take advantage of feature locality (e.g. nonzero indices are dependent). Relations between data features, as in feature locality, may be abstractly encoded as a sampling graph in which each feature represents a node and an edge between two features (an adjacency) permits an atom in $\mc{A}$ with nonzero weights on both of those features. In Figure \ref{fig:intuition_manifold_sampling}, if no two features are adjacent (none related), we recover \Rf (Fig 1a). If all features are pairwise adjacent (all combinations are possible), we recover \Sporf (Fig 1b). At a high-level, \Mf\ introduces user-defined networks in between those of \Rf\ and \Sporf\ allowing for some relations but not others (Fig 1c).

    \begin{figure}[!htbp]
        \centering
        \includegraphics[width=0.75\linewidth]{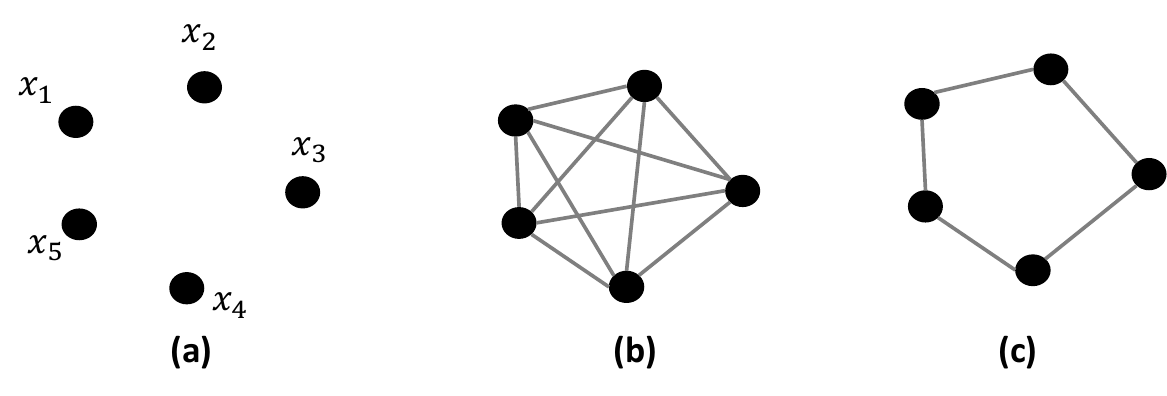}
        \caption{Intuition of manifold sampling for non-zero indices shown with a 5-dimensional data sample, $x = [x_1, x_2, x_3, x_4, x_5]^T \in \mathbb{R}^5$. Sampling non-zero indices can be abstractly represented as a graph. In axis-aligned random forests (a), at every node, samples from the dictionary are drawn where only one non-zero weight is drawn (i.e. the standard basis vector). This is represented by a completely disconnected graph;. In \Sporf\ (b) any combination of non-zero indices is possible, indicated by a fully-connected graph; each sample can effectively form a linear combination of potentially all features. In practice, the sampled projection vector is sparse, which can be represented by the edge weights being extremely small ($\frac{1}{2s}$). In \Morf\ (c) prior information about the structure of $x$ can be leveraged to constrain the dictionary of possible projection vectors. For example, in natural images, we expect adjacent pixels to be correlated and $x_1, x_2, x_5$ may represent adjacent pixels in a vectorized image. Thus, we would sample non-zero indices from these "patches" of the image.}
        \label{fig:intuition_manifold_sampling}
    \end{figure}
    
\paragraph{Selection of Nonzero Weights} In \Rf, all non-zero weights have value 1. In \Sporf, all non-zero weights have value 1, or -1. In our settings \Morf, all non-zero weights will have value 1. For the setting of natural images, this is equivalent to taking the summation operator of a small "patch". These weights can be set in a more general fashion. For example, one can set weights according to a Gaussian kernel, where values are higher in the center of the patch and lower towards the edge of the patch. We do not consider these cases, but discuss their implications in Section \ref{sec:discussion}.
 
\subsubsection{Examples of Projection Dictionaries for Manifolds}
\label{subsubsec:example_projection_manifolds}
    For our applications, we provide a concrete implementation targeted at translation equivariant feature locality in 1D and 2D, such as time series and images respectively. Based on assumptions about the data manifold, one can specify the graph to sample non-zero indices. Assume there exists a similarity matrix, $S$, induced by the manifold dictating how features are related to each other in a data sample, $x_i$. For example, in natural images, $S$ would be a matrix in $\RR^{W \times H}$ with width, W, and length, H. So randomly sample $i \in \{1,...,W\}$ and $j \in \{1,...,H\}$. Then $S_{ij}$ is the ijth entry of the similarity matrix, corresponding to the ijth feature in $x_i$. We would then randomly sample k-hop neighbors of $S_{ij}$ to form a "patch", and then combine these features using the summation operator (these features can be combined with different weights, which is considered in the Discussion). This patch vectorized would form our projection vector, $a_i$. $\langle x_i^T, a_i \rangle$ would give us a candidate feature value to split on. In practice we do not actually know the similarity matrix induced by the manifold, but we can leverage prior information. For example, in natural images, nearby pixels One can specify the possible widths and heights of sampled patches based on knowledge of the data manifold.
    
    
    In natural images, at each split node, a set of atoms are randomly sampled to produce candidate features across observations. \Mf\ accepts hyperparameters defining the minimum and maximum number of patch rows $\{h_{min}$, $h_{max}\}$ and columns $\{w_{min}, w_{max}\}$, respectively. To sample a patch, first the number of rows $h$ and columns $w$ are independently and uniformly sampled between respective minima and maxima (inclusive). As columns are contiguous, a reference leftmost column in the unraveled matrix is sampled as $u \sim \mathcal{U}\{-w+1, W\}$. If 2D locality is specified, then a reference upper row is sampled as $v \sim \mathcal{U}\{-h+1, H\}$. In both cases, the reference column and row may be outside of the matrix so that each feature has an equally likely chance of being included in a patch. The region outside of the matrix is ignored, effectively a zero-padded boundary. The algorithm pseudocode is equivalent to that of \Sporf\ except for the distribution $f_\mc{A}$ described above which can be seen in Appendix \ref{appendix:pseudo}.
    
    In multivariate time-series, the features of a single observation $x_i \in \RR^p$ can also be viewed as organized into a 2D matrix in $\RR^{S \times T}$, where $S * T = p$. Each feature is indexed by a row and column over the number of sensors (S) and time points (T). In this case, only 1D locality might be beneficial where contiguous time points are correlated. However, unless the sensors are ordered in a meaningful manner, there is no reason to suspect locality along the sensor axis (i.e. columns). Again the parameters of the patch can be specified according to the data manifold.

    In our experiments, the atom weights were limited to values of $1$ and $0$ to limit combinatorial complexity but domain-specific atom design may be desired in some settings along with further task-specific atoms. For graph-valued data, one may consider sampling a collection of neighboring edges or nodes \cite{wu_comprehensive_2019}. For spatial related data, one can consider sampling a colleciton of nearby points in Euclidean or Riemannian space \cite{Saha2021-spatialrf}.

\subsection{Feature Importance}

    One of the benefits to decision trees is that their results are fairly interpretable in that they allow for estimation of the relative importance of each feature. Many approaches have been suggested \cite{Breiman2001, Lundberg2017AUA}, and here we introduce a projection forest specific metric which counts the the number of times a given feature was used in projections across the ensemble. Formally, a decision tree $T$ in the trained forest $\texttt{F}$ contains a set of split nodes, where each node $s \in T$ is associated with an atom $a_s^*$ from the dictionary of atoms in the forest $\mathcal{A}_\texttt{F}$ and a threshold that partition the feature space according to the projection $a_s^{*T} x$. Thus, the indices corresponding to nonzero elements of $a_s^*$ indicate important features used in the projection. The importance of feature $k$, denoted $\pi_k$, is calculated as
    
    \begin{align*}
        \pi_k = \frac{1}{|\mathcal{A}_\texttt{F}|} \sum_T \sum_{s \in T_S} \mathbb{I}(a_{sk}^* \neq 0),
    \end{align*}
    
    the number of times it is used in a projection, across all decision trees and split. These counts represent the relative importance of each feature in making a correct classification. Such a method applies to \Rf,\Sporf, and \Mf\ although different results between them would be expected due to different dictionary distributions.
\section{Theoretical Results}
Random forest algorithms have been historically difficult to analyze theoretically, both from a statistical perspectives as well as algorithmic. However, there is a large body of literature making assumptions and modifications on top of Breiman's random forest algorithm \cite{Breiman2001} from which theoretical analyses are tractable \cite{wager_adaptive_2016, meinshausen_quantile_2006, biau_consistency_nodate, denil_narrowing_nodate, biau_analysis_2012, lin_random_2006, scornet_consistency_2015, wager_estimation_2018}. Here, we provide insights on oblique forests, such as \Sporf\, and \Mf\, by expanding upon the axis-aligned forest statistical results in \citet{athey_generalized_2018} and algorithmic results in \citet{louppe2014understanding}.

\subsection{Classifier Consistency}
\citet{athey_generalized_2018} present a seminal paper specifying some minor distributional assumptions and algorithmic conditions for their generalized random forest algorithm to provide a consistent estimate $\hat{\theta}(x)$ of $\theta(x)$, where $\theta(x)$ is defined as the solution to some estimating equation $\EE[\psi_{\theta(x)}(Y_i) | X_i = x] = 0$ for all $x \in \mathcal{X}$ and $\psi_{\theta(x)}$ is a score function. A consistent estimate converges to the true estimand in probability, as the sample size $n$ approaches infinity.

Fundamentally, the consistency of the generalized random forest comes from each tree partitioning the feature space into a set of hyper-rectangles (a bijective map with the set of tree leaves) whose radii go to zero as the sample size grows but slow enough such that they are populated with sets of sizes approaching infinity. This is the same logic behind the consistency of the k-nearest neighbors classifier \cite{dgl} and indeed random forests are effectively adaptive nearest neighbor classifiers \cite{lin_random_2006}.

Although oblique decision trees do not create hyper-rectangular partitions, they do partition the feature space into a finite number of (possibly unbounded) convex polytopes (see Appendix \ref{appendix:hyper}). Each polytope region responds to a leaf node with a constant classification label per leaf. That oblique trees yield convex polytopes which are not necessarily hyper-rectangles is the only difference compared to axis-aligned trees, and the basis as to why the consistency results of \citet{athey_generalized_2018} can be extended to the oblique setting.

\phantomsection\label{subsec:consis}
We show that a main result of \citet{athey_generalized_2018} holds for oblique random forests, such as \Sporf\, and \Mf. As a corollary, posterior probability estimates are consistent. Thus, an oblique random forest under appropriate conditions admits a consistent classification rule and so its error converges to the minimum expected (Bayes) error (see Appendix \ref{appendix:proofs} for full proofs).  We repeat Specification \ref{spec:athey} made by \citet{athey_generalized_2018} below for reference. Specification \ref{spec:oblique} is new and restricts the set of possible dictionaries. The full proofs are detailed in Appendix \ref{appendix:proofs} along with the technical and minor distributional Assumptions 1A-6A.

\paragraph{Specifications}
\begin{enumerate}
    \item \label{spec:athey} All trees are symmetric, in that their output is invariant to permuting the indices of training examples; make balanced splits, in the sense that every split puts at least a fraction $\textit{w}$ of the observations in the parent node into each child, for some $\textit{w} > 0$; and are randomized in such a way that, at every split, the probability that the tree splits on the j-th feature is bounded from below by some $\alpha > 0$. The forest is honest and built via subsampling with subsample size $s$ satisfying $s/n \rightarrow 0$ and $s \rightarrow \infty$ \cite{athey_generalized_2018}.
    
    \item \label{spec:oblique} The oblique dictionary $\mathcal{A}$ is finite and contains the set of standard basis vectors $\{e_i\}_{i=1}^p$, each with a fixed nonzero probability of being selected at each split node.
\end{enumerate}

\paragraph{Assumptions}
\begin{enumerate}
    \item \label{asmptn:dep} There exists a density $f$ over $\mathcal{X}$ that is bounded away from zero and infinity. That is, for all $x \in \mathcal{X}$ there exists a $\varepsilon > 0$ such that $\varepsilon < f(x) < \frac{1}{\varepsilon}$.
    \item \label{asmptn:lipschitz} For all $y \in \mathcal{Y}$, $P(Y = y \mid X = x)$ is Lipschitz continuous in $x \in \mathcal{X}$.
\end{enumerate}

\textit{Honesty}, introduced in Specification \ref{spec:athey}, is a mild condition that removes bias from the leaf estimates by requiring the set of training examples used to learn the structure of the tree to be independent of the set of examples used at the leaf nodes for estimation \cite{biau_consistency_nodate,denil_narrowing_nodate,wager_adaptive_2016,athey_generalized_2018}. In practice this is done using a holdout set per tree and can be beneficial in some cases \cite{wager_estimation_2018}. Alternatively, this sample splitting can be performed more naturally using the out-of-bag samples from bootstrapping. With the addition of Specification \ref{spec:oblique}, the following theorem extends the results of \citet{athey_generalized_2018} to oblique regression forests in addition to axis-aligned forests.

\begin{theorem}\label{theorem:consis}
Under Assumption \ref{asmptn:dep} and Assumptions 1A-6A \cite{athey_generalized_2018}, the estimate $\hat{\theta}(x)$ from the generalized random forest of \citet{athey_generalized_2018} incorporating oblique splits and built to Specifications \ref{spec:athey}-\ref{spec:oblique} is consistent, i.e. $\hat{\theta}_n(x) \overset{P}{\rightarrow} \theta(x)$ as $n \rightarrow \infty$.
\end{theorem}

The conditional mean $\theta(x) = \EE[Y_i | X_i = x]$ is a valid estimand for the generalized random forest algorithm and the empirical estimator coincides with that of Breiman's regression forest \cite{athey_generalized_2018}. So, in the classification setting one may readily estimate the class-conditional posterior in terms of a conditional mean $\theta(x, y) = \EE[\II[Y_i = y] | X_i = x] = P(Y_i = y | X_i = x)$ for all $y \in \mathcal{Y}$ in each leaf node. We note that this choice of estimand satisfies Assumptions 2A-6A, see Appendix \ref{appendix:proofs} for details, with only Assumption 1A remaining as a true assumption. Thus, from Theorem \ref{theorem:consis} we obtain the following corollary and classification Theorem with Assumption 1 incorporated explicitly.

\begin{corollary}\label{corollary:consis}
Under Assumptions \ref{asmptn:dep}-\ref{asmptn:lipschitz},  posterior estimates from an oblique classification random forest built to Specifications \ref{spec:athey}-\ref{spec:oblique} are consistent, i.e. $\hat{p}_n(x; y) \overset{P}{\rightarrow} P(Y_i = y | X_i = x)$ as $n \rightarrow \infty$.
\end{corollary} 

\begin{theorem}\label{theorem:consis_class}
Under Assumptions \ref{asmptn:dep}-\ref{asmptn:lipschitz}, the classification rule from a oblique classification random forest built to Specifications \ref{spec:athey}-\ref{spec:oblique} is consistent, i.e. $L_n  \overset{P}{\rightarrow} L^*$ as $n \rightarrow \infty$.
\end{theorem}

Note that these results apply to both oblique forests with unstructured atoms such as \Sporf\, as well as those with structured atoms such as \Mf. This Lipschitz assumption is a frequent one taken in the literature on random forests \cite{athey_generalized_2018}. It intuitively makes sense \textit{a priori} that small deviations in $x$ should lead to small deviations in the class probability. As in \citet{athey_generalized_2018}, however, this theorem is limited to continuous-valued features which rules out certain classes of data.

\subsection{Training Time Complexity}\label{sec:complexity}
Theoretical analyses are difficult without making assumptions on the data as a tree's possible structure occupies a combinatorially large space and the worst case is too large to be helpful in typical scenarios. We extend the work of \citet{louppe2014understanding} and examine a simplified setting in which the possible sizes induced at each partition node are equally likely. It has been posited and supported empirically that this is a lower bound for the true average case in a \Rf\, \cite{louppe2014understanding}. One reason that worse-than-average cases may occur is that when none of the candidate features are informative, edge splits are frequent and lead to deep trees \cite{louppe2014understanding}. The candidate features in \Sporf\, and \Mf\, are combinations of individual features and we expect this greater flexibility to reduce the chance that no candidate features are informative.

Let a forest have $T$ trees, $d$ candidate features at each split node, and $n$ training samples. At each split node in \Rf, the complexity is $O(dn \log n)$ to sort observations along each feature using an optimal sorting algorithm \cite{louppe2014understanding}. In the average case, the time complexity for \Rf\, is then $\mathcal{O}(Tdn \log^2 n)$ \cite{louppe2014understanding}. \Mf, like \Sporf, utilizes sparse matrix multiplication to compute weighted sum while sorting observations at each split node. Thus, letting $H$ and $W$ denote the maximum height and width of a patch in \Mf, the time complexity for \Mf\, is $\mathcal{O}(Td H W n \log^2 n)$, because for each of the $d$ features we must make $HW$ multiplication operations. However, as we will show empirically, \Mf\, can find better partitions and thus learn smaller trees. 

It is worth noting that because our analysis is based on the framework devised by \cite{athey_generalized_2018}. The assumptions are rather strict and it is not necessary in all cases for example that the joint density of the explanatory variables is uniformly bounded away from zero and infinity. Nevertheless, our results show that structured oblique methods, such as \Mf\ fit into the theory posed by Athey et al. Moreover, the theory posed by Athey et al. are the basis of the heavily used Generalized Random Forest method and code. Although the study of \Mf\ 's convergence behavior is beyond the scope of this paper, empirically it seems that \Mf\ is able to learn less complex trees compared to RF.

\section{Simulation Experiments}
We examine the performance of \Mf\ in terms of predictive accuracy and runtime in three simulations highlighting 1D and 2D manifolds. In all cases. \Mf\ outperforms methods that do not consider feature locality as well as ConvNets in some cases.

\subsection{Three simulated manifolds}
\label{sec:sims_abc}
    We evaluated \Mf\ in three simulation settings to show its ability to take advantage of structure in data. \Mf\ was compared to a set of traditional classifiers (and \Sporf) that learn from the raw features. For each experiment, we used our open source implementation of \Mf\, as well as \Sporf\, and the \Rf\, implementation contained in the \Sporf, each with 500 trees. Other classifiers were run from the Scikit-learn Python package \cite{scikit-learn} and the gradient boosted tree XGBoost (XGB) was run using its Python implementation \cite{xgboost_2016}. Additionally, we tested against a Convolutional Deep Network (ConvNet) built using PyTorch \cite{paszke2017automatic} with two convolution layers, ReLU activations, and maxpooling, followed by dropout and a densely connected hidden layer.
        
    Method hyper-parameters were left as defaults except for the ConvNets and \Mf\, which are each specific to the structure of the data and so must be changed. Thus, a well-performing ConvNet architecture was selected and \Mf\, minimum and maximum patch sizes were optimized using a grid search over a range of potential values as well as well as the number of features considered per split (\textit{mtry}) which should vary with the patch size. See Appendix \ref{appendix:hyper} for details on the hyperparameters and network architectures across experiments.

    Experiment (A) is a non-Euclidean cyclic manifold example inspired by \citet{Younes2018DiffeomorphicL} in which the discriminating information is solely contained in the structure of the data. Each observation is a discretization of a circle's perimeter into a one dimensional feature vector with 100 features and two non-adjacent segments of 1’s in two differing patterns: class 1 features two segments of length five, while class 2 features one segment of length four and one of length six. Because features are arranged on a circle,  segments can wrap around the cyclic feature vector. Figure \ref{fig:6panel}(A) shows examples from the two classes and classification results across various sample sizes.
    
    Experiment (B) is a simple $28\times28$ binary image classification problem. Images in class 0 contain randomly sized and spaced \textit{horizontal} bars while those in class 1 contain randomly sized and spaced \textit{vertical} bars. For each sampled image, $k \sim \text{Poisson}(\lambda = 10)$ bars were distributed among the rows or columns, depending on the class. The distributions of the two classes are identical if a 90 degree rotation is applied to one of the classes and so a classifier cannot be learned without learning the structure of the data. Figure \ref{fig:6panel}(B) shows examples from the two classes and classification results across various sample sizes.
    
    Experiment (C) is a signal classification problem highlighting the presence of structure in time series data. One class consists of 100 values of Gaussian noise independent and identically distributed (iid) while the second class has an added exponentially decaying unit step ($u$) beginning at time 20. 
    
    \begin{align*}
        X_t^{(0)} &= \epsilon\\
        X_t^{(1)} &= u(t-20)\exp^{(t-20)} + \epsilon, \qquad   
        \epsilon \iid \mathcal{N}(0,1)
    \end{align*}
    
    Figure \ref{fig:6panel}(C) shows examples from the two classes and classification results across various sample sizes.
    
    \begin{figure}[!htb]
        \centering
        \includegraphics[width=\linewidth]{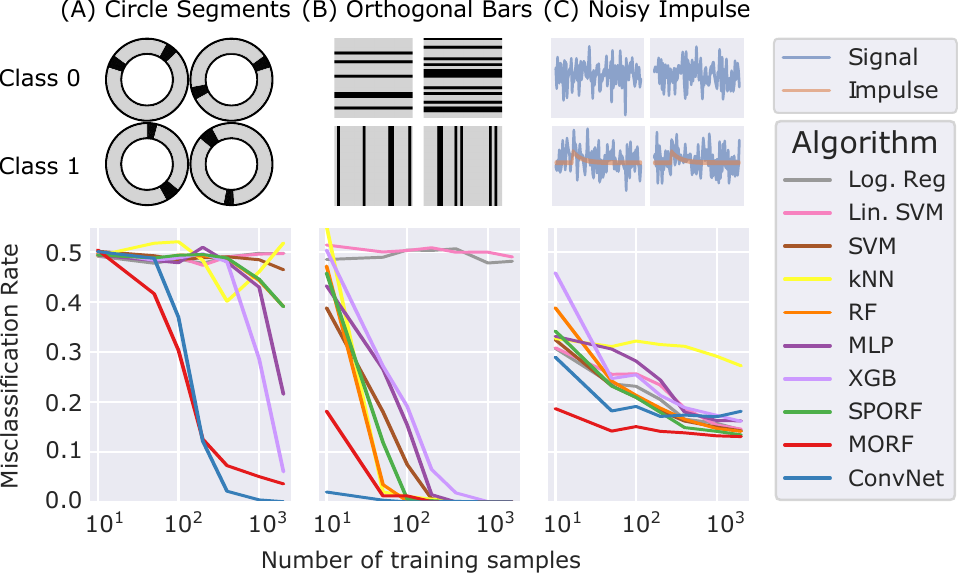}
        \caption{\Mf\, outperforms other algorithms in three two-class classification settings when considering a small number of samples. Upper row shows examples of simulated data from each setting and class. Lower row shows misclassification rate in each setting, tested on 10,000 test samples. \textbf{(A)} Two segments in a discretized circle. Segment lengths vary by class. \textbf{(B)} Image setting with uniformly distributed horizontal or vertical bars. \textbf{(C)} White noise (class 0) vs. exponentially decaying unit impulse plus white noise (class 1). We also observe that ConvNets perform generally better for very large sample sizes (e.g. $>>10^3$), as expected.}
        \label{fig:6panel}
    \end{figure}
    
    In all three simulation settings, \Mf\, outperforms all other classifiers that ignore the local structure, doing especially better at low sample sizes. As compared with ConvNets, \Mf\ sometimes does better, and sometimes worse. The variance across five repeated runs was negligible across sample sizes. The performance of \Mf\ and ConvNets are particularly good in the discretized circle simulation for which most other classifiers perform at chance levels. \Mf\ dominates in the signal classification problem for all sample sizes, most likely because of the ability to learn wide patches which approaches the Bayes classifier. Results on these experiments using uniformly distributed atom weights in between 0 and 1 showed no improvement but increased training time and so were omitted.
    
    We compare the empirical complexity of \Mf, \Sporf, and \Rf\, in Figure \ref{fig:complex_plots}. Although projection forests required more computations at each partition node during training, as outlined in Section \ref{sec:complexity}, \Mf\, is able to learn less complex trees in all cases and sample sizes.

    \begin{figure}[!htb]
        \centering
        \includegraphics[width=\linewidth]{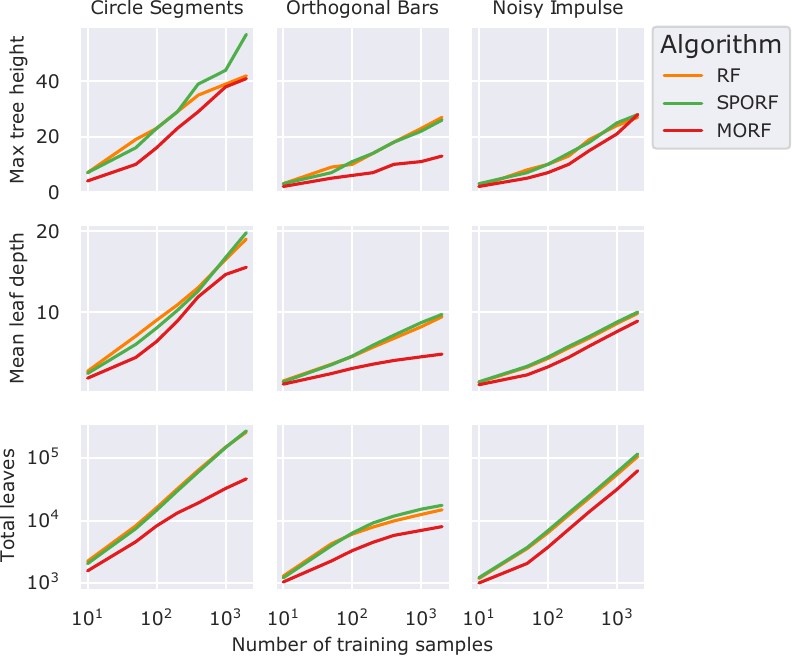}
        \caption{Empirical random forest complexities in each simulation with respect to their maximum tree height, mean leaf depth, and total number of leaves across training sample size. \Mf\, is able to learn simpler trees in all cases due to its restricted projection distribution.}
        \label{fig:complex_plots}
    \end{figure}
    
    Each simulated experiment was run on CPUs and allocated 45 cores for parallel processing. The resulting train and test times as a function of the number of training samples are plotted in Figure \ref{fig:time_plots}. \Mf\, has train and test times slightly longer than those of \Sporf. This cost comes at the benefit of less complex trees, as show in in Figure \ref{fig:complex_plots}. The other method to utilize feature locality, the ConvNet, took noticeably longer to run across simulations for the majority of sample sizes, as seen in Figure \ref{fig:time_plots}. Thus its strong performance in those settings comes at an added computational cost, a typical issue for deep learning methods \cite{dl_cost}.

    \paragraph{Model Misspecification} In the three simulation settings, we explored how \Mf\ is robust to misspecified manifold structure. For example, in (A) Circle Segments, we know that the differentiating factor is the length of the segments being five, or not five. Therefore, one would suspect that the correct patch dimensions should cover that case sufficiently. In Supplementary Figure \ref{fig:morf_model_misspecification}, we see that \Mf\ is relatively stable even when the patch dimensions are not fully the same. The important factor is making sure the information relevant to the task is contained within the possible patch. For full details of the model misspecification experiment, see Supplementary Section \ref{supp_sec:model_miss}.

\subsection{A simulated multivariate time-series}
\label{sec:multivariate_sims}
    Experiment (D) here demonstrates how multivariate data with implicit structure can be learned by \Mf. The simulation is run as in the prior simulations in Section \ref{sec:sims_abc}. We simulate a multivariate time-series problem, where the signals are governed by a linear dynamical system of the form
    
        $$x(t+1) = Ax(t) + B_i u(t),$$
        
    \noindent
    where the governing linear state matrix, $A \in \mathbb{R}^{3 \times 3}$, is the same, but the class separation is modeled by the input matrices, $B_i \in \mathbb{R}^{3 \times 3}$, defined below. The input to the system, $u(t)$, is the same for all classes. The system is set up such that the pair $(A, B)$ is controllable, and that $A$ is marginally stable (i.e. eigenvalues, $|\lambda| \le 1$ for all eigenvalues of A.
    
    \begin{align*}
        B_0 = \begin{pmatrix} 0.5 && 0 && 0 \\ 0 && 0.5 && 0 \\ 0 && 0 && 0.5 \end{pmatrix} \\
        B_1 = \begin{pmatrix} 0.5 && 0.1 && 0.1 \\ 0.1 && 0.5 && 0.1 \\ 0.1 && 0.1 && 0.5 \end{pmatrix}
    \end{align*}
    
    An input $u(t)$ was applied at a random time point selected between the 20th and 40th time point of the simulated time series. A total of 100 time points were simulated and \Mf\ was compared to a suite of other classification algorithms over varying sample sizes. Note that the signals have the exact same dynamics encoded through the $A$ matrix, but the input-output relationships are different. This simulates a setting common in multivariate time-series classification (e.g. EEG, see Section \ref{sec:intracranial_eeg}) where there are signals collected over time that one hypothesizes to be relevant to a task, yet the researcher does not know \textit{a priori} what signals are actually relevant and so they collect as much data as possible. This results in the standard curse-of-dimensionality. However, it is assumed that signals relevant to the task live on a low-dimensional manifold, and the goal is to have a model learn the structure of this manifold for the sake of classification. Figure \ref{fig:sim_multivariate} shows examples from the two classes and classification results across various sample sizes. 

    \begin{figure}[!htb]
        \centering   
        \includegraphics[width=\linewidth]{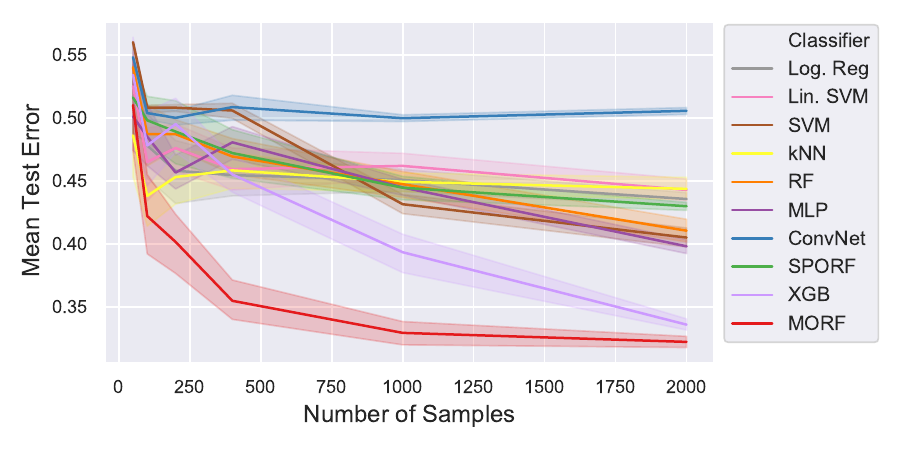}
        \centering
        \caption{\textbf{Multivariate data embedded in a manifold} Algorithm comparisons on classifying two classes of multivariate data. The data consists of samples from two classes of a 3-dimensional stable linear dynamical system with input. Samples are constructed as $X_i \in \mathbb{R}^{3 \times T}$, where now data points over time are correlated. \Mf\ learns the structure significantly faster then the other classifiers, with \Xgboost\ requiring more sample to achieve the same test error rate.}
        \label{fig:sim_multivariate}
    \end{figure}

    A challenging aspect of experiment (D) is that i) the linear state dynamics governed by the A matrix are the same and is considerably larger in norm compared to $B_i$, and ii) the time at which $u(t)$ is applied is random within a small interval. This simulation setting motivates settings where there is a dynamical system with input, such as electroencephalogram (EEG) at rest with dynamics modeled as a linear system, and then a stimulus is applied in the form of input  $u(t)$ \cite{Kerr2017,Li,Li2017,Jones2019}. Depending on the stimulus applied, this might affect the system in different ways through $B_i$. This is very general setting in where the stimulus can be a flash of light, or indication of movement, or even a direct stimuli to evoke seizures.

\section{Experiments on Real Data}
We next evaluated \Mf\ on three real data sets with varying manifold structure, sample sizes and classification goals. In each dataset, there is the notion of either a 1D
or 2D
manifold on which we have \textit{a priori} knowledge of feature locality, time and images respectively. We compare results against a suite of classification algorithms as before.

\subsection{2D Locality: MNIST Digit Classification}
\Mf's performance was evaluated on the MNIST dataset, a collection of handwritten digits stored in 28 by 28 square images \cite{mnist}, and compared to the algorithms used in the simulations. 10,000 images were held out for testing and the remaining 50,000 images were used for training. The results are displayed in Figure \ref{fig:mnist} (top). Hyperparameters are as described in the simulations, see Appendix \ref{appendix:hyper} for details. \Mf\, showed an improvement over the other algorithms as compared to ConvNets.

\begin{figure}[!htb]
    \centering
    \begin{subfigure}[t]{\linewidth}
        \centering
        \includegraphics[width=\textwidth]{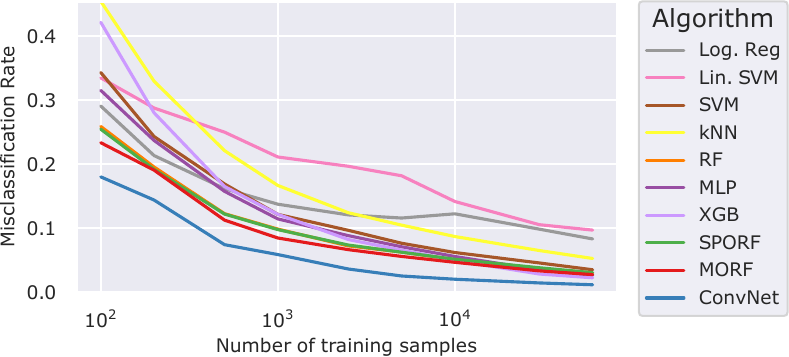}
    \end{subfigure}%
    \hfill
    \begin{subfigure}[t]{\linewidth}
        \centering
        \includegraphics[width=\textwidth]{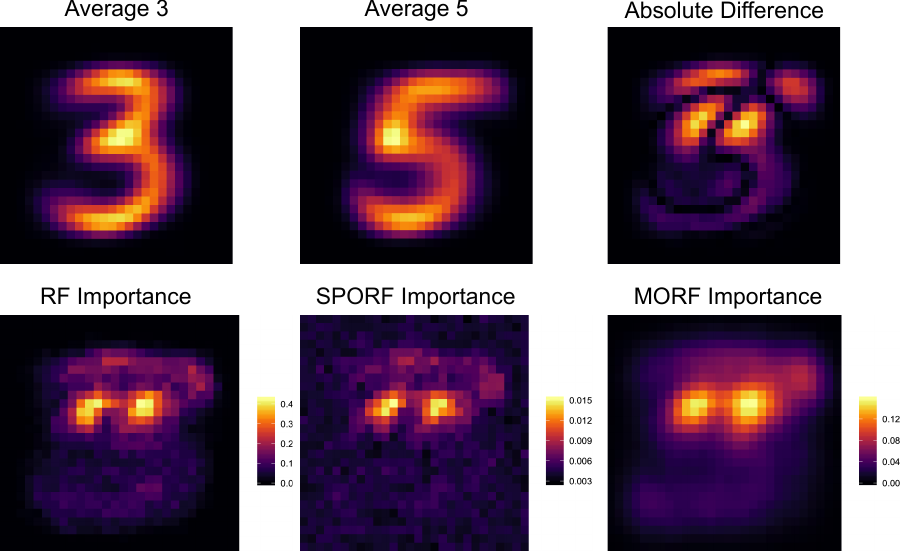}
    \end{subfigure}
    \caption{(\textbf{Top}) \Mf\, performance on the MNIST digit classification problem improves prediction accuracy over all other non-ConvNet algorithms, notably in small sample sizes. (\textbf{Bottom}) The averages all images from MNIST labeled 3 and 5, respectively, and their absolute difference (top row). Feature importance from \Mf\, (bottom right) shows less noise than \Sporf\, (bottom middle) and is smoother than RF (bottom left).}
    \label{fig:mnist}
\end{figure}

We then evaluated the ability of \Mf\ to identify important features in manifold-valued data as compared to \Sporf\ and \Rf. All methods were run on a subset of the MNIST dataset: we only used threes and fives, 100 images from each class.

The feature importance of each pixel is shown in Figure \ref{fig:mnist} (bottom). \Mf\, visibly results in a smoother pixel importance, a result most likely from the continuity of neighboring pixels in selected projections. Although \citet{SPORF} demonstrated empirical improvement of \Sporf\, over \Rf\, on the MNIST data, its projection distribution yields scattered importance of unimportant background pixels as compared to \Rf. Since projections in \Sporf\, have no continuity constraint, those that select high importance pixels will also select pixels of low importance by chance. This may be a nonissue asymptotically, but is a relevant problem in low sample size settings. \Mf\,, however, shows little or no importance of these background pixels by virtue of the modified projection distribution. 
\subsection{1D Locality: Multivariate Time-Series EEG}
\label{sec:intracranial_eeg}\
\Mf's performance was next evaluated on multivariate time-series. Just using raw EEG data, we used \Mf\ to classify movement direction. Compared to a suite of other classification algorithms, \Mf\ is able to achieve a superior performance measured by AUC relative to the other classifiers, as seen in Supplementary Figure \ref{fig:movement_prediction}. We observe that ConvNets have a 0.51 $\pm$ 0.04 AUC, indicating that it overfit to the training data. This is most likely due to the fact that the dataset presented in \cite{Kerr2017} consists of 100-200 trials of data (i.e. samples). We next demonstrate that with some structured feature engineering, one can improve \Mf\ on difficult problems beyond that of any traditional classifier and is superior to ConvNets in low-sample size settings.

We next look at a 91 epilepsy subject dataset \cite{Li862797,7963378,Li2021fragilitybeforeafter} comprised of intracranial electroencephalogram (iEEG) recordings of multiple seizures. Clinicians annotated a subset of implanted electrodes as part of the clinically hypothesized epileptogenic zone (EZ), and then perform subsequent surgery to resect a super set of those regions.
The classification task is to predict surgical outcome of success (seizure free) or failure (seizure recurrence) after a surgical resection is performed on drug resistant epilepsy patients conditioned on the clinically hypothesized EZ regions. If a feature is informative, then it will highly correlate with the clinical EZ when a surgical outcome is successful and vice versa when not successful. In \citet{Li862797,7963378}, a feature of the data, ``neural fragility'' was computed from the data, which is represented as a spatiotemporal heatmap of channels-by-time.

The classification task specifically takes in a data points that are $X_i \in \mathbb{R}^{H \times W}$ of dimension $(20 \times 105)$, where there are 20 quantiles summarizing a distribution of neural fragility and 105 time points around seizure onset. There are 10 quantiles for the neural fragility of electrodes in the clinically hypothesized EZ and 10 quantiles for the neural fragility of electrodes in the rest of the implanted electrodes. The goal is to take $X_i$ and predict $y_i \in \{0, 1\}$, where 0 stands for failed surgical outcome and 1 stands for successful surgical outcome. The classification is setup this way because the clinically annotated EZ electrodes are imperfect and not always representative of the true underlying EZ, which is not observable. Moreover, the multivariate EEG time-series were transformed in this way in order to faciliate comparisons of predictions among subjects with different number of electrodes (ranging from \~30-150). In the original paper, there is preprocessing of the data in the form of a thresholding step, which was done to improve the model performance. However, in this setting, we perform no preprocessing before the classification task to demonstrate fair performance on the "raw" transformed data. For full details on the dataset and clinical problem, we refer the readers to \citet{Li862797,7963378}. 

In Figure \ref{fig:epilepsy_prediction}a, we compare \Mf\ with default hyperparameters, specified in the \Sporf\ package, against standard classification algorithms as in the simulations. The sample sizes for training are relatively low with 60\% of subjects used for training and the rest for the held-out test set. The set of subjects in each cross-validation index are the same across all classifiers, thus enabling a fair comparison. In some folds, it is seen that all the classifiers perform very poorly. This occurs most likely because the data are noisy, the implanted iEEG electrodes are not perfect and some subjects are very difficult to treat. Moreover, there are only 91 subjects total used in this classification task. Even in this challenging classification setting, we observe that \Mf\ is able to achieve a superior performance measured by AUC (Figure \ref{fig:epilepsy_prediction}b). In terms of the Cohen's effect size, \Mf\ is significantly (p-value $\leq 0.05$) more accurate than all other algorithms besides \Sporf\ per a Wilcoxon paired sign test across the 10 cross-validation folds. We observe that ConvNets perform at chance level on the test set, completely overfitting to the training set in this limited sample size setting.

\begin{figure}[!htb]
    \begin{subfigure}[t]{0.48\textwidth}
        \centering   
        \includegraphics[width=\linewidth]{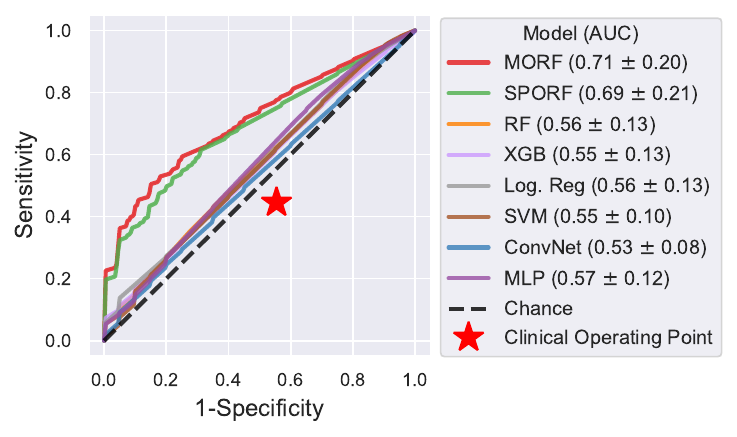}
    \end{subfigure}
    \begin{subfigure}[t]{0.44\textwidth}
        \centering   
        \includegraphics[width=\linewidth]{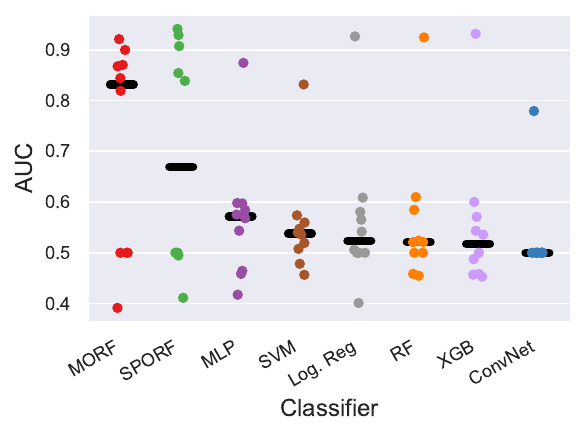}
    \end{subfigure}
    \centering
    \caption{\Mf\ performance on the epilepsy seizure outcome prediction problem improves prediction accuracy over all other algorithms in 10-fold cross validation (CV). \textbf{(Top)} Shows a ROC curve with the mean ROC curve plotted AUC for each classifier. \textbf{(Bottom)} Shows a strip plot of the AUC values in 10-fold CV with the median marked in each setting (solid black line). Note that even the ConvNet and multi-layer perceptron (MLP) perform poorly, most likely because the sample size is very low. Compared to the next best non \Sporf\ classifier, \Mf\ improves over the MLP with a Cohen's D effect size of 0.83 (95\% CI = [2.32, -0.102]).}
    \label{fig:epilepsy_prediction}
\end{figure}
\section{Discussion}
\label{sec:discussion}
The success of sparse oblique projections in decision forests has opened up many possible ways to improve axis-aligned decision forests (including random forests and gradient boosting trees) by way of specialized projection distributions. Traditional decision forests have already been applied to some manifold-valued data, using predefined features to classify images or pixels, and have shown great success, but ignore feature continuity and specialize for specific data modalities. We expand upon sparse oblique projections and introduced manifold-aware projection distributions that exploits prior knowledge of the local topology of a feature space to improve learning rates and accuracy for classification. The open source implementation of \Mf\ subsumes \Sporf\ and provides a flexible classification method for a variety of data modalities and tailored projection dictionaries. We showed in various settings that appropriate domain knowledge can improve the projection distribution and better match ConvNet results (or even outperforming ConvNets significantly) while maintaining interpretability, fast run time, and theoretical justification. 

It is plausible that one could design a loss, or engineer a feature based on the structure of one’s dataset. Then presumably recent extensions of RF would work even better than \Mf\ on structured data (e.g. data that is spatially dependent) \cite{Saha2021-spatialrf}. However, this explicit structure is often not known in practice and hence incorporating dependence structure directly within the loss function is not as useful for an off-the-shelf tool. \Mf\ circumvents this issue by sampling projection vectors from a dictionary and using a recursive surrogate loss instead (i.e. the Gini impurity per leaf). Moreover, \Mf\ is significantly cheaper in terms of computational cost since we are not directly performing optimization on a desired loss function (e.g. see Supplementary Figure S1 on training and testing times).

The flexibility in choices of \Mf's dictionary opens a much larger combinatorial space to sample from compared to a traditional random forest. More complex possibilities may lead to improved performance, but potentially at the cost of greater sampling requirements. Similarly, research into other task-specific projection dictionaries may lead to improved results in computer vision tasks, through better texture quantification for instance, or in other manifold-valued settings such as graphs. Although, unlike ConvNets, \Mf\ is not globally translation equivariant, it can be locally translation equivariant given atoms reminiscent of Gabor filters, for instance. Without local equivariance, discriminative features must be constant in their indices or the training data must be rich enough to fully encapsulate possible observations \cite{kinect_rf}. The \Mf\ projection distributions may also be incorporated into other state of the art Forest algorithms such as \sct{XgBoost}. Additionally, the fact that oblique decision forests lead to partitions of convex polytopes is of interest in that it has been shown that deep nets with Rectified Linear Units (ReLUs) or hard tanh activation layers also partition the feature space into convex polytopes with different linear functions on each region \cite{raghu_expressive_2017}. This shared ``partition and vote scheme'' offers insight into their relationship with one another as well as the functioning of the brain \cite{priebe_modern_2020} 

\bibliography{references}
\bibliographystyle{IEEEtranN}

\clearpage

\beginsupplement
\appendix
\section*{Acknowledgements}
This work is supported by the Defense Advanced
Research Projects Agency (DARPA) Lifelong Learning Machines program through contract FA8650-18-2-7834 and through funding from Microsoft Research. AL is supported by NIH T32 EB003383, the NSF GRFP (DGE-1746891), the Arcs Chapter Scholarship, Whitaker Fellowship and the Chateaubriand Fellowship. The authors have no conflicts of interest to declare.

\section*{Appendices}

\section{Proofs}\label{appendix:proofs}
\subsection{Convex Polytope Partition Results}
As mentioned in Section \ref{subsec:consis}, a random projection tree partitions the feature space into a finite number of (possibly unbounded) convex polytopes. The proof of that is as follows.

\begin{proof}
A convex polytope in $d$ dimensions can be defined as the union of a finite number of halfspaces, where a halfspace is a $d-1$ dimensional surface defined by the linear inequality
\begin{equation*}
    a^Tx \leq b
\end{equation*}
for fixed $a \in \RR^d$ and $b \in \RR$.
In a random projection tree, each split node $i$ partitions the set of points at that node according to such an inequality $a_i^Tx \leq b_i$. Consider the path of $k$ split nodes, including the root, to a leaf $l$ and the set of corresponding halfspace defining $\{(a_i,b_i\}_{i=1}^k$ terms for each split node. We see that in the feature space $S$, the subset that "falls into" leaf $l$ is the solution set to
\begin{equation*}
    A_lx \leq b_l
\end{equation*}
where $A_l = [a_1, \dots ,a_k]^T$ and $b_l = [b_1,\dots,b_k]^T$.

Thus each leaf node forms a convex polytope. Additionally, note that any $x \in S$ will deterministically end up in a leaf node (by classification of $x$) as the tree is of finite depth and that all leaf node convex polytopes are mutually exclusive as the lowest common ancestor of any two leaves forms mutually exclusive sets. If the feature space is unbounded, then at least one partition must be unbounded too. Thus, a tree partitions the feature space into a finite number of possibly infinite convex polytopes.
\end{proof}

\subsection{Consistency Results}
The least we can ask of our classification rule $\{g_n\}_{n=1}^{\infty}$ is for it to be consistent,
$L_n  \overset{P}{\rightarrow} L^*$ as $n \rightarrow \infty$, where $L_n$ and $L^*$ are the expected 0-1 losses of the finite sample rule $g_n$ and the Bayes decision rule $g^*$, respectively. Our results here build upon the results of \citet{athey_generalized_2018} who prove under Assumption \ref{asmptn:dep}, 1A-6A and Specification \ref{spec:athey} that their generalized random forest (GRF) algorithm, which subsumes Breiman's regression forest, provides a consistent estimate $\hat{\theta}_n(x)$ for some quantity $\theta(x)$. The estimand $\theta(x)$ is defined as the solution to the generic estimating equation $M_\theta(x) := \EE[\psi_{\theta}(Y_i) | X_i = x] = 0$ for all $x \in \mathcal{X}$ where $\psi_{\theta}$ is a score function. We begin by restating Assumptions 1A-6A of the GRF algorithm as they are relevant to further results but strongly recommend referring to the original paper \citep{athey_generalized_2018} for additional details.

\paragraph{Assumptions:}
\begin{enumerate}
    \item[1A.] For fixed values $\theta(x)$, we assume that $M_\theta(x)$ is Lipschitz continuous in $x$.
    \item[2A.] When $x$ is fixed, we assume that $M_\theta(x)$ is twice continuously differentiable  in $\theta$ with a uniformly bounded second derivative, and that $\frac{\partial}{\partial(\theta)} M_\theta(x) \mid_{\theta} \neq 0$ for all $x \in \mathcal{X}$.
    \item[3A.] The worst-case variogram of $\psi_\theta (Y)$ is Lipschitz-continuous in $\theta(x)$.
    \item[4A.] The $\psi$-functions can be written as $\psi_{\theta}(Y) = \lambda(\theta(x); Y) + \xi_{\theta}(g(Y))$, such that $\lambda$ is Lipschitz-continuous in $\theta$, $g: \{Y\} \rightarrow \R$ is a univariate summary of $Y$, and $\xi_{\theta} : \R \rightarrow \R$ is any family of monotone and bounded functions.
    \item[5A.] For any weights $\alpha_i(x)$ such that $\sum_i \alpha_i(x) = 1$, the estimation equation returned a minimizer $\hat{\theta(x)}$ that at least approximately solves the estimating equation $||\sum_{i=1}^n \alpha_i(x) \psi_{\hat{\theta}} (Y_i)||_2 \leq C max\{\alpha_i(x)\}$ for some constant $C \geq 0$.
    \item[6A.] The score function $\psi_{\theta} (Y)$ is a negative sub gradient of a convex function, and the expected score $M_\theta(x)$ is the negative gradient of a strongly convex function.
\end{enumerate}

\paragraph{Proof of Theorem \ref{theorem:consis}: Consistent oblique random forests posterior estimates}

Our Theorem \ref{theorem:consis} extends Theorem 3 of the GRF paper \citep{athey_generalized_2018} given the additional Specification \ref{spec:oblique}.
Specifically, the GRF Theorem 3 in \citet{athey_generalized_2018}  result relies on the consistency result of Theorem 1 in \citet{wager_estimation_2018}. We need only to verify foundations of Theorem 1 \citep{wager_estimation_2018} that take into account the use of axis-aligned splits, those being Theorems 3 and 5 of \citet{wager_estimation_2018}. Thus, it suffices to confirm that those results are unchanged under oblique splits and given Specification \ref{spec:oblique} holds.

Theorem 3 \citep{wager_estimation_2018} proves an asymptotic upper bound on the diameter of a leaf, diam$(L(x))$, by applying an asymptotic upper bound result from Lemma 2 \citep{wager_estimation_2018} on the diameter of dimension $j$ in the leaf, diam$_j(L(x))$. The leaf $L(x)$ is a polytope formed from a combination of axis-aligned and oblique splits. Considering only the axis-aligned conditions forming $L(x)$, by the positive probability of splitting on each dimension per Specification \ref{spec:oblique}, the upper bound of Lemma 2 \citet{wager_estimation_2018} holds. As the addition of oblique conditions cannot increase the size of the leaf, the same upper bound holds. Similarly, the diameter diam$(L(x))$ of the leaf is smaller than the diameter of the polytope formed from just axis-aligned conditions. By the diameter bound from Lemma 2 \citep{wager_adaptive_2016} of each feature, the upper bound of Theorem 3 \citep{wager_estimation_2018} holds for the axis-aligned polytope and so also $L(x)$.

Theorem 5 \citep{wager_estimation_2018} hinges on Lemma 4 \citep{wager_estimation_2018} which brings up the concept of a \textit{potential nearest neighbor} (PNN) \citep{wager_estimation_2018, wager_adaptive_2016}.

\begin{defi}
$x_i \in \{x_1, \dots, x_s\} \in \{\RR^p\}^s$ is a \textit{potential nearest neighbor} (PNN) of $x \in \RR^p$, if there is an axis-aligned hyperrectangle containing only $x$ and $x_i$. A $k$-PNN set is a collection of $k$ points and $x$ in an axis-aligned hyperrectangle containing no other points. A predictor $T$ for $x$ is a $k$-PNN predictor if given
\begin{equation*}
    \{z\} = \{(x_1,y_1),\dots,(x_s,y_s)\} \in \{\RR^p \times \mathcal{Y}\}^s,
\end{equation*}
$T$ outputs the average of the $y_i$ among a $k$-PNN set of $x$ with respect to the $x_i$.
\end{defi}

In the case of oblique split decision trees, we have the following result.

\begin{lem}
\label{lem:pnn}
Let $T$ be a decision tree which makes oblique splits (including axis-aligned splits) at each interior node with finite dictionary $\mc{A}$ of $m$ vectors encoding the set of allowable oblique axes. If $T$ has leaves between size $k$ and $2k-1$, then $T$ is a $k$-PNN predictor on $\RR^m$.
\end{lem}
\begin{proof}
Let $\mathcal{X}$ denote the vector space of possible samples, where $x \in \mathcal{X} \subset \RR^p$. Since $\mathcal{A} \in \{\RR^p\}^m$, let $A \in \RR^{p \times m}$ denote the matrix whose columns are the elements of $\mathcal{A}$. Then $\mathcal{B} = A^T\mathcal{X} \subset \RR^m$ is a vector space of dimension at most $\min(p, m)$ in a space of dimension $m$. Bases of $\mathcal{B}$ correspond to bases or oblique combinations of them from $\mathcal{X}$ and so every oblique split in $\mathcal{X}$ is an axis-aligned split in $\mathcal{B}$. The points which fall into a leaf of $T$ are the only points which satisfy the linear system formed by the set of splits, which are the only points that fall into the hyperrectangle in $\mathcal{B}$ defined by that system. As any decision tree making axis-aligned splits with leaves of sizes between $k$ and $2k - 1$ is a $k$-PNN predictor \citep{lin_random_2006}, $T$ is thus a $k$-PNN predictor in $\mathcal{B} \subset \RR^m$.
\end{proof}

In this expanded feature space from which we can view oblique splits as axis-aligned, as in the above proof, we can scale down the marginals to be within $[0,1]$. While this space no longer satisfies the Lipschitz criteria and if $m > p$ may have a density of $0$ at all points outside of the $p$ dimensional subspace, Lemma 4 \cite{wager_estimation_2018} requires neither of these conditions from the original assumptions. So it holds, albeit with the finite constant $m$ instead of $p$. Thus Theorem 5 \citep{wager_estimation_2018} holds with simply a modified constant which doesn't change the final established asymptotics in Theorem 1 \citep{wager_estimation_2018}. So Theorem 1 \citep{wager_estimation_2018} holds in our oblique forest setting and we can extend Theorem 3 of \citet{athey_generalized_2018} to the oblique setting per our Theorem \ref{theorem:consis}.

\paragraph{Proof of Corollary \ref{corollary:consis}: Consistent posterior probability estimates}
In the spirit of \citet{perry_random_2021}, we now prove the corollary of our previous Theorem that the consistent regression estimate results of \citet{athey_generalized_2018} extend to classification and so oblique random forests produce consistent estimation rules $\{p_n(y \mid x)\}_{n=1}^\infty$ of the posterior $P(Y = y \mid X = x)$ which we denote as $p(y \mid x)$. It is important to note that consistency proof is independent of the splitting mechanism at each leaf node and so switching to the Gini impurity score and away from the mean squared error score of the GRF only has the potential to affect convergence rates.

To show that the posterior probability estimates are consistent, we need to show that $p_n(y \mid x) \overset{P}{\rightarrow} {p}(y \mid x)$ as $n \rightarrow \infty$. Let $y$ be the fixed arbitrary class label of interest. Given our data, we seek to estimate the class-specific posterior probability $p(y \mid x)$, equivalent to estimating the conditional mean $\theta(x) := \EE [\II[Y = y] \mid X = x]$ for any $y \in \mathcal{Y}$. For conciseness, in the following proofs we will often drop the notational dependence of $\theta$ on $x$ and $y$ where convenient, letting it implicitly be a function of any fixed pair $(x, y) \in \mathcal{X} \times \mathcal{Y}$. To follow the notation of \citet{athey_generalized_2018}, we frame $\theta(x)$ as the solution to the estimation equation
\begin{equation*}
   M_\theta (x) := \EE[\psi_{\theta}(Y) \mid X = x] = 0
\end{equation*}
where the score function $\psi_{\theta}(Y)$ is defined as
\begin{equation*}
    \psi_{\theta}(Y) := \II[Y = y] - \theta(x).
\end{equation*}
This estimand $\theta(x)$ can be estimated by solution $\hat{\theta}(x)$ to the empirical estimation equation
\begin{equation*}
    \sum_{i=1}^n \alpha_i(x) \psi_{\hat{\theta}}(Y_i) = 0.
\end{equation*}
It follows that 
$$\hat{\theta}(x) = \sum_{i=1}^n \alpha_i(x) \II[Y_i = y]$$
per the expansion
\begin{align*}
    \sum_{i=1}^n \alpha_i(x) \psi_{\hat{\theta}} (Y_i)
    &= \sum_{i=1}^n \alpha_i(x) (\II[Y_i = y] - \hat{\theta}(x)) \\
    &= \sum_{i=1}^n \alpha_i(x)\II[Y_i = y] - \hat{\theta}(x) = 0.
\end{align*}

These weights we learn from a learned random forest. Let a forest be composed of $B$ trees. In a single tree $b$, let $l_b(x)$ denote the set of training examples at the leaf node for which $X$ is placed. Define the weights $\alpha_{ib}(x)$ for that tree as
\begin{equation*}
    \alpha_{ib}(x) := \frac{1}{|l_b(x)|} \II[x_i \in l_b(x)],
\end{equation*}
the normalized indicator of whether or not $x$ and $x_i$ exist in the same leaf.  Thus the forest weights $\alpha_i(x) = \frac{1}{B}\sum_{b=1}^B \alpha_{ib}(x)$ are simply the normalized weights across all trees.

By Theorem 3 of \citet{athey_generalized_2018}, a random forest built according to Specification \ref{spec:athey} and solving an estimation problem satisfying Assumptions 1A-6A yields a consistent estimator. We enumerate these assumption, defined above, and verify that they hold for posterior probability estimate $\hat \theta(x)$.

\begin{enumerate}
    \item[1A.] This remains a true assumption on the distribution and so is restated as Assumption \ref{asmptn:lipschitz}.
    \item[2A.] This is true, as evident in the derivatives
        \begin{equation*}
            \frac{\partial}{\partial\theta} M_\theta(x) = -1
            \quad\text{and}\quad
            \frac{\partial^2}{\partial^2\theta} M_\theta(x) = 0.
        \end{equation*}
    \item[3A.] This is evident in the worse-case variogram for two solutions $\theta(x)$ and $\theta'(x)$ 
        \begin{align*}
            &\gamma (\theta(x), \theta'(x))
            \\&:= \sup_{x \in \mathcal{X}} \{Var(\psi_{\theta}(Y) -\psi_{\theta'}(Y) \mid X = x)\}
            \\&= \sup_{x \in \mathcal{X}} \{Var(\theta'(x) - \theta(x) | X=x)\}
            = 0
        \end{align*}
        which is trivially Lipschitz-continuous.
    \item[4A.] Clearly $\psi_{\theta}(Y)$ is linear in $\theta(x)$ and so is a Lipschitz-continuous function in $\theta(x)$. The other term is 0 in this case.
    \item[5A.] As shown previously shown, the estimation equation is solved to equal 0.
    \item[6A.] This holds true by construction of the convex function $\Psi_{\theta}(Y) := \frac12 (\II[Y=y \mid X = x] - \theta(x))^2$ such that $\psi_\theta(Y) = -\frac{d}{d\theta} \Psi_{\theta}(Y)$, and the strongly convex function $\mathbb{M}_\theta(x) := \frac12 (P(Y=y \mid x) - \theta(x))^2$ where $\quad M_\theta(x) = -\frac{d}{d\theta} \mathbb{M}_\theta(x)$.
\end{enumerate}

This verifies Assumptions 2A-6A from Theorem \ref{theorem:consis} for the finite-sample estimate $\hat p_n(y | x) := \hat \theta(x)$ of $\theta(x) := p(y | x)$. Corollary \ref{corollary:consis} follows, adding Assumption 1A to the required Specifications \ref{spec:athey}-\ref{spec:oblique} from Theorem \ref{theorem:consis}.

\paragraph{Proof of Theorem \ref{theorem:consis_class}: A consistent classification rule.}

Corollary \ref{corollary:consis} established consistency for each posterior probability estimate $p_n(y \mid x)$. We now proceed to show consistency for the classification rule $g_n(x) = \argmax_{y} p_n(y \mid x)$. As before, define $p(y \mid x) := P(Y = y | X = x)$

\begin{lem}
\label{lem:unequal}
Let $x\in \mathcal{X}$ with true, but unknown, unique maximum $y^* := \argmax_y p(y \mid x)$, and define the finite sample estimate $\hat{y} := \argmax_y p_n(y \mid x)$. If $p_n(y \mid x)$ is a consistent estimator for $p(y \mid x)$, then
\begin{equation*}
    P[\hat{y} \neq y^* \mid x] \rightarrow 0 \quad\text{as}\quad n \rightarrow \infty
\end{equation*}
\end{lem}
\begin{proof}
We omit the conditional for notational brevity by substituting $p(y) := p(y \mid x)$ and $p_n(y) := p_n(y \mid x)$. Then it follows that
\begin{align*}
    P[\hat{y} \neq y^* \mid x] &= 
    P[\max_y p_n(y) > p_n(y^*)]\\
    &=P\left[\bigcup_{y\neq y^*} p_n(y) > p_n(y^*)\right]\\
    &\leq\sum_{y \neq y^*} P[p_n(y) > p_n(y^*)]\\
    &=\sum_{y \neq y^*} P[p_n(y) - p_n(y^*)  > 0]\\
    &=\sum_{y \neq y^*} P[(p_n(y) - p_n(y^*) ) - \\
    &\quad\quad(p(y) - p(y^*)) > p(y^*) - p(y)]
\end{align*}
Let $\varepsilon_y := p(y^*) - p(y)$ and note that $\varepsilon_y > 0$ for all $y \in \mathcal{Y} \setminus \{y^*\}$ since $y^*$ is a unique maximum. Observe that
\begin{align*}
    &\sum_{y \neq y^*} P[(p_n(y) - p_n(y^*) ) - (p(y) - p(y^*)) > \varepsilon_y]\\
    &\leq \sum_{y \neq y^*} P\big[\big|(p_n(y) - p_n(y^*) ) - (p(y) - p(y^*))\big| > \varepsilon_y \big]
\end{align*}
By the consistency of the individual posteriors, the difference of two is consistent and so since $\mathcal{Y}$ is a finite set,
\begin{align*}
    &P[\hat{y} \neq y^* \mid x] \\
    &\leq\sum_{y \neq y^*} P[|(p_n(y) - p_n(y^*) ) - (p(y) - p(y^*))| > \varepsilon_y]\\
    &\rightarrow 0 \quad\text{as}\quad n \rightarrow \infty
\end{align*}
\end{proof}

With Lemma \ref{lem:unequal}, the proof of Theorem \ref{theorem:consis_class} follows. Denote the finite samples classification rule $\hat{y} := \argmax_y p_n(y \mid x)$ as before and let $y^* := \argmax_y p(y \mid x)$ be a unique maximum. If $y^*$ were not unique, we would instead consider the aggregate of all such maximum classes as a pseudo class, apply the following analyses, and be confident in both $L_n$ and $L^*$ up to a factor equal to the reciprocal of the number aggregated classes due to a chance guess between them.

Otherwise, for any $\varepsilon > 0$, by the law of total probabilities,
\begin{align*}
    &P\left[|L_n - L^*| > \varepsilon\right] \\
    &=P\left[|p_n(\hat{y} \mid x) - p(y^* \mid x)| > \varepsilon \right] \\
    &=P\left[|p_n(\hat{y} \mid x) - p(y^* \mid x)| > \varepsilon \mid \hat{y} = y^* \right] \times P\left[\hat{y} = y^*\right] \\
    &+P\left[|p_n(\hat{y} \mid x) - p(y^* \mid x)| > \varepsilon \mid \hat{y} \neq y^* \right] \times P\left[\hat{y} \neq y^*\right].
\end{align*}
In the case that $\hat{y} = y^*$, by Corollary \ref{corollary:consis} we have convergence of the posteriors and so
\begin{equation*}
    P\left[|p_n(\hat{y} \mid x) - p(y^* \mid x)| > \varepsilon \mid \hat{y} = y^* \right] \rightarrow 0 \quad\text{as}\,\, n\rightarrow \infty.
\end{equation*}
In the case that $\hat{y} \neq y^*$, by Lemma \ref{lem:unequal} we have that
\begin{equation*}
    P\left[\hat{y} \neq y^*\right] \rightarrow 0 \quad\text{as}\,\, n\rightarrow \infty.
\end{equation*}
Since the probabilities are bounded above by one, it follows that as $n \rightarrow \infty$,
\begin{equation*}
    P\left[|p_n(\hat{y} \mid x) - p(y^* \mid x)| > \varepsilon \mid \hat{y} = y^* \right] \times P\left[\hat{y} = y^*\right] \rightarrow 0
\end{equation*}
and
\begin{equation*}
    P\left[|p_n(\hat{y} \mid x) - p(y^* \mid x)| > \varepsilon \mid \hat{y} \neq y^* \right] \times P\left[\hat{y} \neq y^*\right] \rightarrow 0
\end{equation*}
and thus
\begin{equation*}
    P\left[|L_n - L^*| > \varepsilon\right] = P\left[|p_n(\hat{y} \mid x) - p(y^* \mid x)| > \varepsilon \right] \rightarrow 0
\end{equation*}

\section{Experiment Extras}\label{appendix:sims}


\subsection{Mathematical description of sampling manifolds in simulation examples}
    In \ref{subsubsec:example_projection_manifolds}

\subsection{Three simulated manifolds: time complexity}
Each simulated experiment was run on CPUs and allocated 52 cores for parallel processing. The resulting train and test times as a function of the number of training samples are plotted in Figure \ref{fig:time_plots}. \Mf\, has train and test times on par with those of \Sporf\, and so is not particularly more computationally intensive to run. The ConvNet, however, took noticeably longer to run across simulations for the majority of sample sizes.
\begin{figure}[!htb]
\centering
\includegraphics[width=0.48\textwidth]{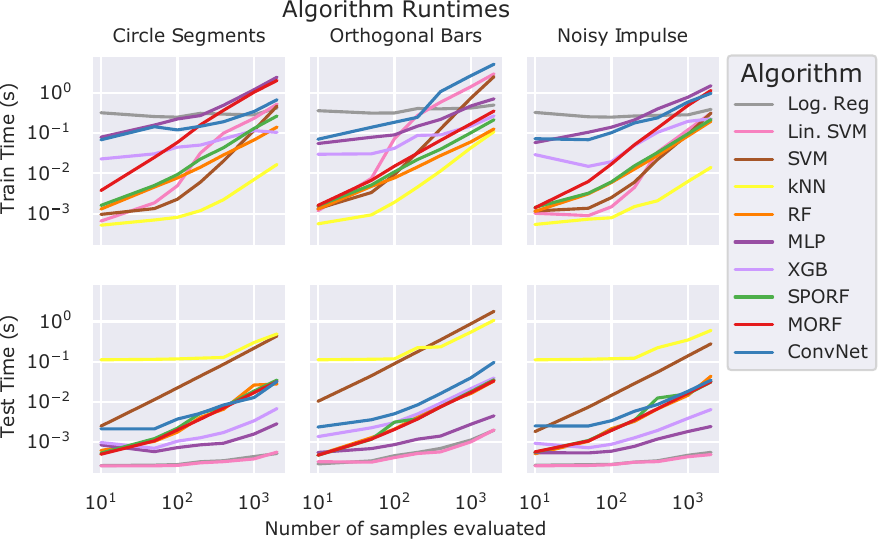}
\caption{Algorithm train times (above) and test times (below) across increasing sample sizes. \Mf\, runtime is not particularly costly and well below ConvNet runtime in most examples.}
\label{fig:time_plots}
\end{figure}

\subsection{Intracranial EEG Experiments - Ethics}
For the motor control, details of the experiment are in \citet{Kerr2017}. If the patient expressed interest in participating, the research staff would verbally review the written, IRB approved consent form. If agreed upon, the patient would sign the written consent and be enrolled in the study. A copy of the written consent would also be given to patient to keep. Experimental protocols were approved by the Cleveland Clinic Institutional Review Board. 

For the epilepsy intracranial EEG data, details on the dataset can be found at \cite{Li862797}. All data were acquired with approval from the local institutional review board (IRB) at each clinical institution: UMMC by the IRB of the University of Maryland School of Medicine; UMH by the University of Miami Human Subject Research Office—Medical Sciences IRB; NIH by the National Institutes of Health IRB; JHH by Johns Hopkins IRB; and CClinic by the Cleveland Clinic IRB. Informed consent was given at each clinical center. The acquisition of data for research purposes was completed with no impact on the clinical objectives of the patient stay. Digitized data were stored in an IRB-approved database compliant with Health Insurance Portability and Accountability Act regulations.

\subsection{1D Locality: Predicting Movement Direction With Intracranial EEG}
\label{sec:movement_prediction_ieeg}\,

\Mf's performance was next evaluated on stereotactic electroencephalogram (sEEG) data recorded in epilepsy patients undergoing a motor control task presented in \citet{Kerr2017,10.3389/fnins.2019.00715}. The classification task presented here is to predict movement direction (up, down, left, or right) based on the sEEG data alone, rather then performing explicit feature engineering, such as computing power in frequency bands. We compare \Mf\, to other classification algorithms. The interesting aspect of this data is that there are no motor regions recorded. Thus, our hypothesis is that only a subset of the recording electrodes over time are important in decoding movement directionality. This is analogous to Experiment D mentioned in Section \ref{sec:multivariate_sims}. Each subject performed the task for several trials, each consisting of a movement instruction followed by a movement generated by the subject. Using only the sEEG data, we sought to decode the movement directionality. For full details on the dataset and clinical problem, we refer the readers to \citet{Kerr2017,10.3389/fnins.2019.00715}.

Here we perform 5-fold cross validation for each subject including all the sEEG recording electrodes time-locked to a movement onset marking. Each fold is \emph{a priori} generated per subject, where there is a set of  testing trials left out. The overall task is very challenging because there are no motor brain regions being recorded. Nonetheless, we expect that other brain regions are involved in the motor control process. \Mf\ is able to achieve a superior performance measured by AUC relative to the other classifiers, as seen in Figure \ref{fig:movement_prediction}. Notably, \Mf\ and \Sporf\ perform the best in this setting with a limited set of training samples, whereas the ConvNet performs slightly better then chance on the test set, overfitting to training set. Across the set of all folds of all subjects, \Mf\ was never worse than another classifier in terms of the median pairwise difference in Cohen's kappa while being significantly (p-value $\leq 0.05$) better than the MLP and ConvNet per a Wilcoxon paired sign test on those same pairwise differences.

\begin{figure}[!htb]
    \begin{subfigure}[t]{0.48\textwidth}
        \centering   
        \includegraphics[width=\linewidth]{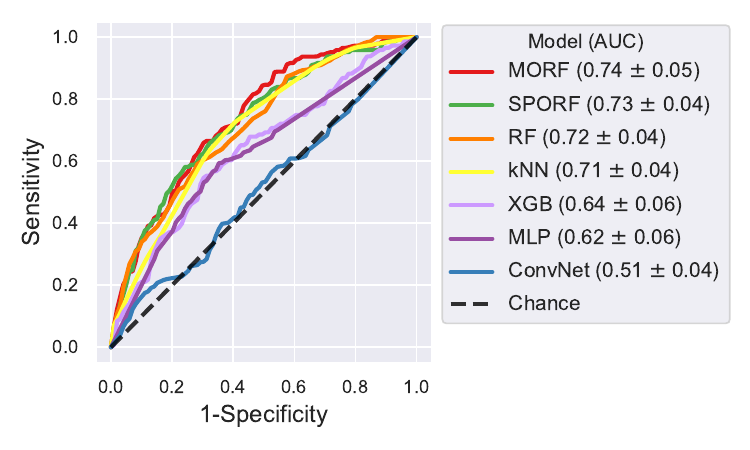}
    \end{subfigure}
    \begin{subfigure}[t]{0.44\textwidth}
        \centering   
        \includegraphics[width=\linewidth]{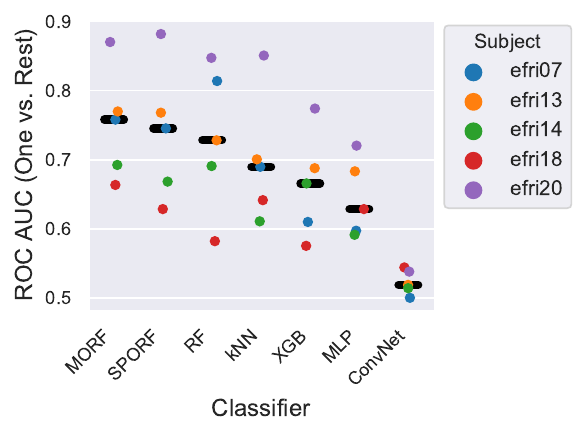}
    \end{subfigure}
    \centering
    \caption{\Mf\ performance on decoding movement direction from the raw sEEG data in non-motor brain regions. Subjects are undergoing a motor-control task. The naming of subjects is simply derived from their clinical monitoring session and does not reflect any specific numbering scheme. \textbf{(Top)} Shows ROC curve of moving down in the motor task decoded with all classifiers on the same set of data and their AUC scores. \textbf{(Bottom)} Shows a summary AUC stripplot where each dot represents the held-out trial median AUC score for a certain subject over 5-fold CV, and the median of the overall AUC for each classifier is shown (solid black line). In almost all subjects, \Mf\ gains in AUC compared to the other classifiers with fixed hyperparameters and fixed trials in each of the 5 folds.}
    \label{fig:movement_prediction}
\end{figure}

\subsection{Three simulated manifolds: model-misspecification}
\label{supp_sec:model_miss}
For each simulation problem, we optimized the optimal parameters for \Mf\ on a large dataset of 2000 samples using grid search over a range of valid parameter values for the manifold structure. Once we arrived at an optimal parametrization, we then proceeded to modify the parametrizations for each of the experiments to change the dimensions of the possible patches \Mf\ could sample. For full details, see the online repository, where the experiment was done (https://github.com/adam2392/morf-demo). We then performed 10-fold stratified cross-validation for each parameter setting, computing accuracy on the held-out test set. We randomly produced 500 samples from each simulated dataset, and kept all other \Mf\ parameters the same. We used a total of 500 trees for each problem. Each non-optimal parametrization was normalized with respect to the accuracy scores of the optimal parametrization (subtracting the mean and dividing by the standard deviation of the optimal scores). The normalized accuracy scores of each non-optimal parameter setting are then shown in Supplementary Figure \ref{fig:morf_model_misspecification}. 

\begin{figure}
    \centering
    \begin{subfigure}[t]{0.32\textwidth}
        \includegraphics[width=\linewidth]{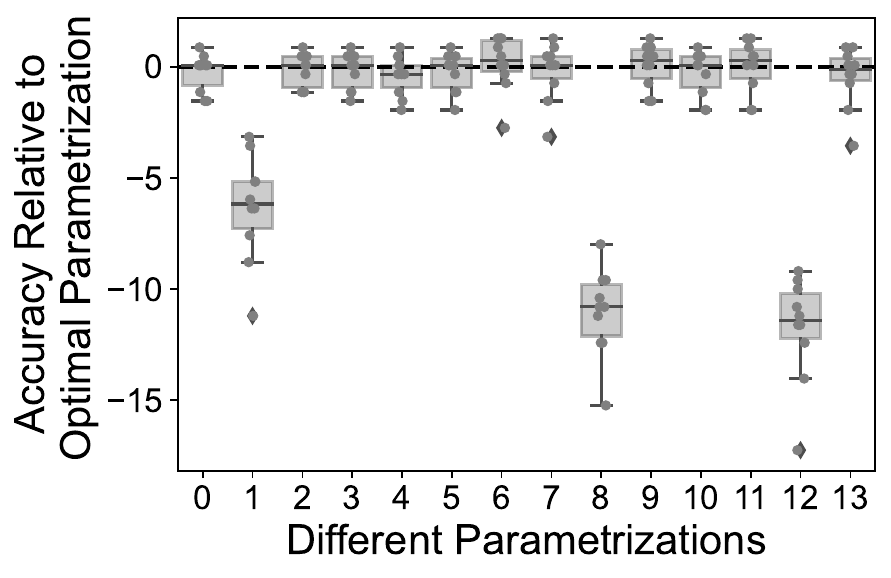} 
        \caption{}
    \end{subfigure}
    \centering
    \begin{subfigure}[t]{0.32\textwidth}
        \includegraphics[width=\linewidth]{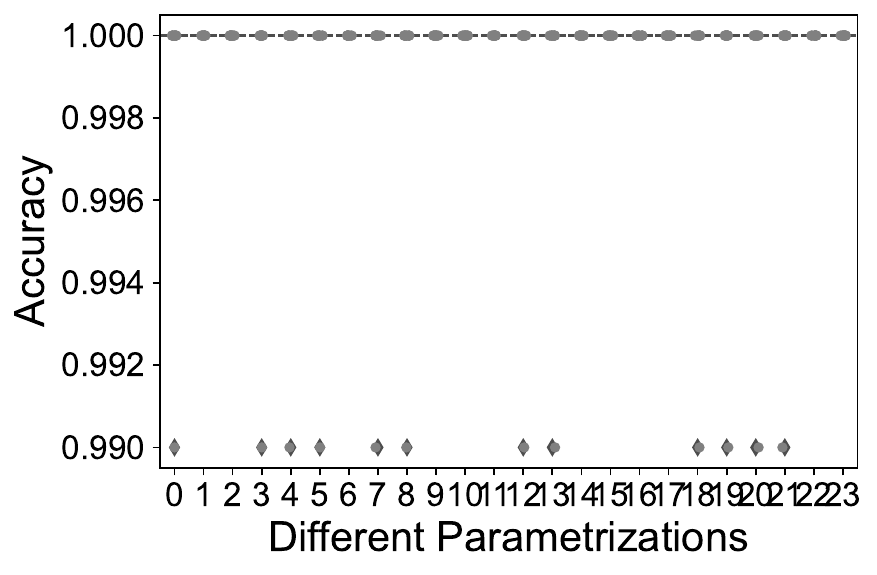} 
        \caption{}
    \end{subfigure}
    \centering
    \begin{subfigure}[t]{0.32\textwidth}
        \includegraphics[width=\linewidth]{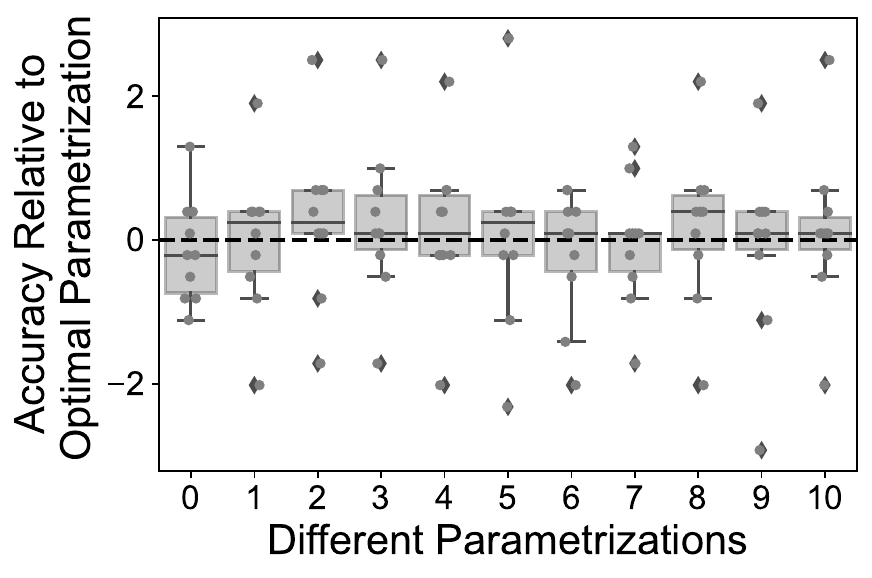}
        \caption{}
    \end{subfigure}
    \caption{\Mf\ performance when model is mis-specified. (a) Circle segments, (b) Horizontal bars, and (c) a 1D time-series with an impulse. The dashed line in (a) and (c) indicate a consistent accuracy score relative to the optimal parametrizations. Each simulation (b) is shown with just accuracy because the relatively low variance in the optimal parameter classifier made the resulting normalized plot uninterpretable. The low performance of a few parameterizations in (a) were when the maximum possible patch size was set to 6. This likely restricts \Mf\ from learning the Circle Segment structure as fast as when \Mf\ can sample larger patches.}
    \label{fig:morf_model_misspecification}
\end{figure}


\clearpage
\section{Pseudocode}\label{appendix:pseudo}
\begin{algorithm}[!ht]
  \caption{Learning a Manifold Oblique decision tree, modified from \citet{SPORF}. 
}
  \label{alg:rerftrain}
\begin{algorithmic}[1]
  \Require (1) $\mathcal{D}_n$: training data (2) $d$: dimensionality of the projected space, (3) $f_{\mathbf{A}}$: distribution of the atoms, 
  (4) $\Theta$: set of split eligibility criteria
  \Ensure A \Mf\, decision tree $T$
  \Function{$T =$ growtree}{$\mathbf{X},\mathbf{y},f_{\mathbf{A}},\Theta$}
  \State $c = 1$
  \Comment{$c$ is the current node index}
  \State $M = 1$
  \Comment $M$ is the number of nodes currently existing
  \State $S^{(c)} =$ bootstrap($\{1,...,n\}$)
  \Comment $S^{(c)}$ is the indices of the observations at node $c$
  \While{$c < M + 1$}
  \Comment visit each of the existing nodes
  \State $(\mathbf{X'},\mathbf{y'}) = (\mathbf{x}_i,y_i)_{i \in S^{(c)}}$
  \Comment data at the current node
  \Linefor{$k = 1,\ldots,K$}{$n_k^{(c)} = \sum_{i \in S^{(c)}} I[y_i = k]$}
  \Comment class counts {(for classification)}
  \If{$\Theta$ satisfied}
  \Comment do we split this node?
  \State $\mathbf{A} = [\mathbf{a}_1 \cdots \mathbf{a}_d] \sim f_{\mathbf{A}}$
  \Comment sample random $p \times d$ matrix of atoms
  \State $\mathbf{\widetilde{X}} = \mathbf{A}^T\mathbf{X'} = (\mathbf{\widetilde{x}}_i)_{i \in S^{(c)}}$
  \Comment random projection into new feature space
  \State $(j^*,t^*) =$ findbestsplit($\mathbf{\widetilde{X}},\mathbf{y'}$)
  \Comment Algorithm \ref{alg:rerfsplit} 
  \State $S^{(M+1)} = \{i: \mathbf{a}_{j^*} \cdot \mathbf{\widetilde{x}}_i \leq t^* \quad \forall i \in S^{(c)}\}$
  \Comment assign to left child node
  \State $S^{(M+2)} = \{i: \mathbf{a}_{j^*} \cdot \mathbf{\widetilde{x}}_i > t^* \quad \forall i \in S^{(c)}\}$
  \Comment assign to right child node
  \State $\mathbf{a}^{*(c)} = \mathbf{a}_{j^*}$
  \Comment store best projection for current node
  \State $\tau^{*(c)} = t^*$
  \Comment store best split threshold for current node
  \State $\kappa^{(c)} = \{M+1,M + 2\}$
  \Comment node indices of children of current node
  \State $M = M + 2$
  \Comment update the number of nodes that exist
  \Else
  \State $(\mathbf{a}^{*(c)},\tau^{*(c)},\kappa^{*(c)}) =$ NULL
  \EndIf
  \State $c = c + 1$
  \Comment move to next node
  \EndWhile
  \State \Return $(S^{(1)},\{\mathbf{a}^{*(c)},\tau^{*(c)},\kappa^{(c)},\{n_k^{(c)}\}_{k \in \mathcal{Y}}\}_{c=1}^{m-1})$
  \EndFunction
\end{algorithmic}
\end{algorithm}

\begin{algorithm}[!ht]
  \caption{As in \citet{SPORF}. Finding the best node split. This function is called by growtree (Alg \ref{alg:rerftrain}) at every split node. For each of the $p$ dimensions in $\mathbf{X} \in \Real^{p \times n}$, a binary split is assessed at each location between adjacent observations. The dimension $j^*$ and split value $\tau^*$ in $j^*$ that best split the data are selected. The notion of ``best'' means maximizing some choice in scoring function. In classification, the scoring function is typically the reduction in Gini impurity or entropy. 
  The increment function called within this function updates the counts in the left and right partitions as the split is incrementally moved to the right.}
  \label{alg:rerfsplit}
\begin{algorithmic}[1]
  \Require (1) $(\mathbf{X},\mathbf{y}) \in \Real^{p \times n} \times \mathcal{Y}^n$, where $ \mathcal{Y} = \{1,\ldots,K\}$
  \Ensure (1) dimension $j^*$, (2) split value $\tau^*$
  \Function{$(j^*,\tau^*) =$ findbestsplit}{$\mathbf{X},\mathbf{y}$}
  \For{$j = 1,\ldots,p$}
  \State Let $\mathbf{x}^{(j)} = (x_1^{(j)},\ldots,x_n^{(j)})$ be the $jth$ row of $\mathbf{X}$.
  \State $\{m_i^j\}_{i \in [n]} =$ sort($\mathbf{x}^{(j)}$)
  \Comment $m_i^j$ is the index of the $i^{th}$ smallest value in $\mathbf{x}^{(j)}$
  \State $t = 0$
  \Comment initialize split to the left of all observations
  \State $n' = 0$
  \Comment number of observations left of the current split
  \State $n'' = n$
  \Comment number of observations right of the current split
  \If{(\text{task is classification})}
  \For{$k = 1,\ldots,K$}
  \State $n_k = \sum_{i=1}^n I[y_i = k]$
  \Comment total number of observations in class $k$
  \State $n'_k = 0$
  \Comment number of observations in class $k$ left of the current split
  \State $n''_k = n_k$
  \Comment number of observations in class $k$ right of the current split
  \EndFor
  \EndIf
  \For{$t = 1,\ldots,n-1$}
  \Comment assess split location, moving right one at a time
  \State $(\{(n'_k,n''_k)\},n',n'',y_{m_t^j}) =$ increment($\{(n'_k,n''_k)\},n',n'',y_{m_t^j}$)
  \State $Q^{(j,t)} =$ score($\{(n'_k,n''_k)\},n',n''$)
  \Comment measure of split quality
  \EndFor
  \EndFor
  \State $(j^*,t^*) = \argmax\limits_{j,t} Q^{(j,t)}$
  \Linefor{$i = 0,1$}{$c_i = m_{t^* + i}^{j^*}$}
  \State $\tau^* = \frac{1}{2}(x_{c_0}^{(j^*)} + x_{c_1}^{(j^*)})$
  \Comment compute the actual split location from the index $j^*$
  \State \Return $(j^*,\tau^*)$
  \EndFunction
\end{algorithmic}
\end{algorithm}

\clearpage
\section{Hyperparameters}\label{appendix:hyper}
\begin{table}[ht]
\caption{ConvNet hyperparameters for each experiment.} 
\centering 
\begin{tabular}{ccl}
\hline
{ Experiment} & { Classifier} & {Architecture Sequence} \\
\hline
{ Circle} & { ConvNet} & {Conv1d(32, window=6, stride=1)} \\
{ } & { } & { MaxPool1d(window=2, stride=2)} \\
{ } & { } & { Conv1d(64, window=10, stride=1)} \\
{ } & { } & { MaxPool1d(window=2, stride=2)} \\
{ } & { } & { Dropout(p=0.5), Linear(500, 2)} \\
{ H/V Bars} & { ConvNet} & { Conv2d(32, window=5, stride=1)} \\
{ } & { } & { MaxPool1d(window=2, stride=2)} \\
{ } & { } & { Conv2d(64, window=5, stride=1)} \\
{ } & { } & { MaxPool1d(window=2, stride=2)} \\
{ } & { } & { Dropout(p=0.5), Linear(200, 2)} \\
{ Impulse} & { ConvNet} & { Conv1d(32, window=10, stride=1)} \\
{ } & { } & { MaxPool2d(window=2, stride=2)} \\
{ } & { } & { Conv2d(64, window=5, stride=1)} \\
{ } & { } & { MaxPool2d(window=2, stride=2)} \\
{ } & { } & { Dropout(p=0.5), Linear(200, 2)} \\
{Experiments D, E} & { ConvNet} & { Conv2d(32, window=2, stride=1)} \\
{ } & { } & { MaxPool2d(window=2, stride=2)} \\
{ } & { } & { Conv2d(32, window=5, stride=1)} \\
{ } & { } & { MaxPool2d(window=2, stride=2)} \\
{ } & { } & { Conv2d(64, window=5, stride=1)} \\
{ } & { } & { MaxPool2d(window=2, stride=2)} \\
{ } & { } & { Linear(64), Linear(n\_classes)} \\
{ MNIST} & { ConvNet} & { Conv2d(32, window=5, stride=1)} \\
{ } & { } & { MaxPool2d(window=2, stride=2)} \\
{ } & { } & { Conv2d(64, window=5, stride=1)} \\
{ } & { } & { MaxPool2d(window=2, stride=2)} \\
{ } & { } & { Dropout(p=0.5), Linear(200, 10)} \\
{Surgical Outcome} & { ConvNet} & { Conv2d(32, window=2, stride=1)} \\
{ } & { } & { MaxPool2d(window=2, stride=2)} \\
{ } & { } & { Conv2d(32, window=5, stride=1)} \\
{ } & { } & { MaxPool2d(window=2, stride=2)} \\
{ } & { } & { Conv2d(64, window=5, stride=1)} \\
{ } & { } & { MaxPool2d(window=2, stride=2)} \\
{ } & { } & { Linear(64), Linear(n\_classes)} \\
{Predicting Movement} & { ConvNet} & { Conv2d(32, window=3, stride=1)} \\
{ } & { } & { MaxPool2d(window=2, stride=2)} \\
{ } & { } & { Conv2d(64, window=3, stride=1)} \\
{ } & { } & { MaxPool2d(window=2, stride=2)} \\
{ } & { } & { Conv2d(64, window=3, stride=1)} \\
{ } & { } & { MaxPool2d(window=2, stride=2)} \\
{ } & { } & { Linear(64), Linear(n\_classes)} \\
\hline
\end{tabular}
\label{table:cnn}
\end{table}

\begin{sidewaystable}[ht]
\caption{\textit{scikit-learn}, \Sporf, and \Mf\, hyperparameters for each experiment.}
\small
\centering
\begin{tabular}{lll}
\hline
Experiment & Classifier & Hyperparameters \\ \hline
{All} & {Lin. SVM} & {C=1,   penalty="l2";kernel="linear";loss="squared\_hinge"} \\
{All} & {Log. Reg} & {C=1, penalty="l2"} \\
{All} & {MLP} & {activation="relu"; alpha=0.0001; hidden\_layer\_sizes=(100,);   solver="adam"} \\
{All} & {RF} & {n\_trees=500, max\_features='sqrt'} \\
{All} & {XGB} & {n\_boosting\_rounds = 10; learning\_rate=0.3; max\_depth=6; subsample=1; tree\_method='auto' (all default)} \\
{All} & {SPORF} & {n\_trees=500, max\_features='sqrt'} \\
{All} & {SVM} & {C=1; gamma=1/(n\_features*Var(X)); kernel="rbf"} \\
{All} & {kNN} & {n\_neighbors=5; p=2} \\
{Circle} & {MORF} & {n\_trees=500, max\_features=0.5; patch\_height\_max=1; patch\_height\_min=1;   patch\_width\_max=12; patch\_width\_min=3} \\
{H/V Bars} & {MORF} & {n\_trees=500, max\_features='sqrt'; patch\_height\_max=2; patch\_height\_min=2;   patch\_width\_max=9; patch\_width\_min=2} \\
{Impulse} & {MORF} & {n\_trees=500, max\_features=0.3; patch\_height\_max=1; patch\_height\_min=1;   patch\_width\_max=12; patch\_width\_min=2} \\
{Experiment D} & {MORF} & {n\_trees=500, max\_features='sqrt'; patch\_height\_max=2; patch\_height\_min=2;   patch\_width\_max=10; patch\_width\_min=5;} \\
{Experiment E} & {MORF} & {n\_trees=500, max\_features='sqrt'; patch\_height\_max=2; patch\_height\_min=1;   patch\_width\_max=20; patch\_width\_min=5;} \\
{MNIST} & {MORF} & {n\_trees=500, max\_features='sqrt' patch\_height\_max=2; patch\_height\_min=2;   patch\_width\_max=5; patch\_width\_min=2;} \\
{Surgical Outcome} & {MORF} & {n\_trees=500, max\_features='sqrt'; patch\_height\_max=sqrt(height); patch\_height\_min=1;} \\  {} & {} & {patch\_width\_max=sqrt(width); patch\_width\_min=1;} \\
\end{tabular}
\label{table:sklearn}
\end{sidewaystable}

\end{document}